\documentclass[twoside,11pt]{article}

\usepackage{blindtext}

\usepackage{jmlr2e}
\usepackage{amsmath}
\usepackage{mathrsfs}
\usepackage{booktabs}
\usepackage{url}

\newtheorem{assumption}{Assumption}

\usepackage{lastpage}

\ShortHeadings{}{A Unified Framework for In-Context Learning}
\firstpageno{1}

\begin{document}
	\title{A Unified Framework for In-Context Learning with Causal and Masked Language Models}
	
	\author{\name Chenrui Liu \email chenruiliu@mail.bnu.edu.cn \\
		\addr Beijing Normal University at Zhuhai\\
		Zhuhai, China
		\AND
		\name Chuanlong Xie \email clxie@bnu.edu.cn \\
		\addr Beijing Normal University at Zhuhai\\
		Zhuhai, China
		\AND
		\name Falong Tan \email falongtan@hnu.edu.cn \\
		\addr Hunan University\\
		Changsha, China
		\AND
		\name Yicheng Zeng \email zengyicheng@mail.sysu.edu.cn \\
		\addr Sun Yat-sen University\\
		Shenzhen, China
		\AND
		\name Lixing Zhu\thanks{Corresponding author.} \email lzhu@bnu.edu.cn \\
		\addr Beijing Normal University at Zhuhai\\
		Zhuhai, China
	}
	
	\editor{My editor}
	
	\maketitle
	
	\begin{abstract}
		In-context learning (ICL) has emerged as a central capability of pretrained
		language models, yet its theoretical analysis has focused primarily on causal
		language models trained by left-to-right autoregressive prediction, such as
		GPT-style models. Masked language models instead recover masked tokens from
		bidirectional context, and their role in ICL remains less understood. We develop
		a statistical learning framework that represents the context examples by their
		empirical measure and models prediction as a function of the context and the
		query. This formulation places autoregressive and masked pretraining objectives
		within a common excess-risk analysis.
		Under Wasserstein-type regularity conditions, we relate pretraining with \(T\)
		tasks and \(N\) samples per task to \(k\)-shot excess risk at inference,
		obtaining same-order upper bounds for masked and autoregressive objectives.
		We also study task-distribution shift, where pretraining tasks are sampled
		from \(\mathbb P\) and inference tasks from \(\mathbb Q\); the resulting bound
		contains an additional term controlled by the lifted Wasserstein distance
		between \(\mathbb P\) and \(\mathbb Q\). The bounds further imply an
		order-optimal allocation under a fixed pretraining data budget and refined
		rates under intrinsic low-dimensional structure. Experiments on controlled function-learning tasks show that the Masked Pair Encoder (MPE) can achieve performance comparable to GPT-2-style causal Transformers, suggesting that ICL behavior is not specific to causal language models.
	\end{abstract}


	\section{Introduction}
	Recent advances in large language models (LLMs) have led to impressive
	performance across a wide range of tasks. One particularly intriguing phenomenon
	exhibited by such models is \emph{in-context learning} (ICL)
	\citep{brown2020language}, whereby a pretrained model adapts to a new task at
	inference time by conditioning on a prompt containing a few input--output
	examples, without updating its parameters.
	
	The ICL phenomenon has motivated substantial theoretical interest in
	understanding its underlying mechanisms
	\citep{garg2022can,von2023transformers,ahn2023transformers,
		bai2023transformers,zhang2024trained}. Some works interpret ICL as a form of
	implicit Bayesian inference
	\citep{xie2022an,muller2022transformers,panwar2024incontext}, while others
	analyze the algorithmic mechanisms underlying ICL by viewing transformers as
	approximators of classical optimization procedures, such as gradient descent
	\citep{von2023transformers}, preconditioned gradient descent
	\citep{ahn2023transformers}, domain adaptation procedures
	\citep{hataya2024automatic}, and Newton-type methods
	\citep{giannou2023looped,fu2024transformers}. From a statistical learning
	perspective, several recent works study the generalization properties of ICL
	and establish excess-risk guarantees
	\citep{li2023transformers,kim2024transformers,wumany,
		ma2025provable,wakayama2025context,liu2025context,ching2026efficient}.

	However, most existing theoretical analyses of ICL have been developed for
	causal language models, such as GPT-style models, whose pretraining is based on
	autoregressive next-token prediction
	\citep{radford2019language,brown2020language}. By contrast, masked language
	models, such as BERT, are trained by masked-token prediction, where masked tokens
	are recovered from the remaining bidirectional context \citep{devlin2019bert}.
	Despite being a central paradigm of language-model pretraining, masked language
	modeling has received comparatively little theoretical attention in the study of
	ICL. This gap is nontrivial because existing analyses for autoregressive ICL do
	not directly apply to masked objectives. In an autoregressive objective, each
	target is predicted from a prefix context, whereas in a masked objective the
	target response is excluded and prediction is made from the remaining
	bidirectional context. Thus, the two objectives provide different context
	information to the predictor. Although recent empirical evidence suggests that
	masked language models can also exhibit ICL behavior \citep{samuel2024berts}, a
	statistical theory of ICL under masked objectives remains largely undeveloped.
	This leads to the question studied in this paper: how do autoregressive and
	masked objectives affect the excess risk and scaling behavior of ICL?
	
	We address this question by developing a unified statistical learning framework
	that connects pretraining objectives to $k$-shot inference. The basic
	abstraction is to separate the two roles played by an ICL prompt: the
	in-context examples provide task-specific information, while the query input
	specifies where prediction is made. We encode the in-context examples by their
	empirical measure and keep the query input as an explicit argument, leading to
	predictors of the form $f(\rho,\mathbf{x})$. This abstraction is aligned with
	the measure-theoretic formulation of attention used in our Transformer
	analysis.
	
	This representation allows masked and autoregressive pretraining to be compared
	within a single predictor class. The two objectives differ only in how the
	empirical context measure is constructed: autoregressive pretraining uses
	prefix measures, whereas masked pretraining uses leave-one-out measures. This
	pretraining-to-inference procedure is not a standard empirical risk
	minimization problem. To account for the finite context window of language
	models, we distinguish the number $N$ of examples available per task during
	pretraining from the number $k \le N$ of in-context examples provided at
	inference. Our analysis therefore tracks how empirical pretraining risk, based
	on $T$ tasks and $N$ examples per task, translates into $k$-shot excess
	risk.
	
	The resulting framework yields a common excess-risk decomposition for masked
	and autoregressive objectives, and leads to same-order excess-risk bounds for
	the two pretraining paradigms. Our contributions are as follows.
	\begin{itemize}
		\item We formulate a unified statistical learning framework for ICL that places
		autoregressive and masked pretraining objectives within a common excess-risk
		analysis. This framework provides a theoretical basis for comparing causal and
		masked language-model pretraining and, to the best of our knowledge, gives the
		first systematic theoretical analysis of ICL under masked pretraining.
		
		\item Under Wasserstein-type regularity conditions, we derive excess-risk
		bounds for autoregressive and masked pretraining objectives. The bounds relate
		pretraining with \(T\) tasks and \(N\) examples per task to \(k\)-shot excess
		risk at inference, and imply both an order-optimal budget allocation and
		refined low-dimensional rates.
		
		\item We establish a \(k\)-shot transferability bound under task-distribution
		shift. Unlike the standard no-shift setting in which both pretraining and
		inference tasks are sampled from the same meta-distribution, we allow
		pretraining tasks to be sampled from \(\mathbb P\) and inference tasks from
		\(\mathbb Q\). The resulting excess-risk bound contains an additional term
		controlled by the lifted Wasserstein distance between \(\mathbb P\) and
		\(\mathbb Q\).
		
		\item We provide a controlled empirical comparison between the Masked Pair
		Encoder (MPE) and a GPT-2-style causal Transformer for in-context function
		learning. Across representative function-learning tasks, MPE achieves
		performance comparable to the causal baseline. These results provide empirical
		support for the view that masked language models can also exhibit ICL.
	\end{itemize}
	
	\paragraph{Notation.}
	Let \(\mathcal X \subseteq \mathbb R^{d_x}\) and
	\(\mathcal Y \subseteq \mathbb R^{d_y}\) be compact sets, and set
	\(\mathcal Z := \mathcal X \times \mathcal Y \subseteq
	\mathbb R^{d_x+d_y}\). We equip \(\mathcal X\), \(\mathcal Y\), and
	\(\mathcal Z\) with the Euclidean norm. Thus, for
	\(\mathbf z=(\mathbf x,\mathbf y)\) and
	\(\mathbf z'=(\mathbf x',\mathbf y')\), one has
	\(\|\mathbf z-\mathbf z'\|_2 =
	(\|\mathbf x-\mathbf x'\|_2^2+\|\mathbf y-\mathbf y'\|_2^2)^{1/2}\).
	In particular, \(\mathcal Z\) is a compact metric space. We write
	\(\mathcal P(\mathcal Z)\) for the space of Borel probability measures on
	\(\mathcal Z\), endowed with the topology of weak convergence. Each task is a
	distribution \(\rho\in\mathcal P(\mathcal Z)\) generating input-output pairs.
	The meta-distribution \(\mathbb P\in\mathcal P(\mathcal P(\mathcal Z))\) is a
	distribution over such task distributions. We use \(W_1\) for the Wasserstein
	distance between task distributions, with respect to the Euclidean norm on
	\(\mathcal Z\), and \(\mathbb W_1\) for the lifted Wasserstein distance between
	meta-distributions. Boldface letters denote vector variables,
	\(\delta_{\mathbf z}\) denotes the Dirac measure at \(\mathbf z\), and
	\([n]:=\{1,\ldots,n\}\) for \(n\in\mathbb N\). For nonnegative quantities
	\(a\) and \(b\), we write \(a\lesssim b\) if \(a\le Cb\) for a constant
	\(C>0\) independent of the relevant sample-size parameters, and write
	\(a\asymp b\) if both \(a\lesssim b\) and \(b\lesssim a\) hold.
	
	\section{Idealized Statistical Model for In-Context Learning}
	\label{sec:idealized_stat_model}
	
	We formalize in-context learning as prediction from a prompt consisting of
	in-context examples and a query input. Let $\mathcal Z=\mathcal X\times\mathcal Y$
	be the example space, and let $\rho\in\mathcal P(\mathcal Z)$ denote a task
	distribution. Conditional on $\rho$, examples
	$\mathbf z_i=(\mathbf x_i,\mathbf y_i)$ are sampled i.i.d.\ from $\rho$.
	
	Fix a maximal context length $N\in\mathbb N$. For $k\in[N]$, a $k$-shot prompt is
	given by $P^{(k)}=(\mathbf z_1,\ldots,\mathbf z_k,\mathbf x_{k+1})$, where
	$\mathbf z_1,\ldots,\mathbf z_k$ are the in-context examples and
	$\mathbf x_{k+1}$ is the query input. The goal is to predict the corresponding
	response $\mathbf y_{k+1}$.

	For statistical analysis, we represent the information contained in the context
	by the empirical measure
	$\widehat\rho^{\,k}=k^{-1}\sum_{i=1}^k\delta_{\mathbf z_i}\in\mathcal P(\mathcal Z)$.
	Thus the prompt is represented by the pair
	$(\widehat\rho^{\,k},\mathbf x_{k+1})$. A related empirical-measure viewpoint
	appears in \citet{mroueh2023towards}. In the present paper, this representation
	places contexts of different lengths in the common space $\mathcal P(\mathcal Z)$
	and makes explicit the two components of in-context prediction: the empirical
	measure $\widehat\rho^{\,k}$ represents the task information supplied by the
	context, while $\mathbf x_{k+1}$ is the query input at which prediction is made.

	An idealized in-context predictor is a measurable map
	$f:\mathcal P(\mathcal Z)\times\mathcal X\to\mathcal Y$. Given a finite context,
	the prediction is $f(\widehat\rho^{\,k},\mathbf x_{k+1})$. For population-level
	analysis, we use the corresponding idealized form $f(\rho,\mathbf x)$, where the
	task is represented by its underlying distribution $\rho$.
	
	To model variation across tasks, let
	$\mathbb P\in\mathcal P(\mathcal P(\mathcal Z))$ be a meta-distribution over task
	distributions. We measure the population performance of an in-context predictor
	$f$ over tasks drawn from $\mathbb P$ by the expected risk
	\begin{equation}
		\label{eq:lifted_risk}
		R_{\mathbb P}^{\ell}(f)
		:=
		\mathbb E_{\rho\sim\mathbb P}
		\mathbb E_{(\mathbf X,\mathbf Y)\sim\rho}
		\left[
		\ell\bigl(\mathbf Y,f(\rho,\mathbf X)\bigr)
		\right]
		=
		\int_{\mathcal P(\mathcal Z)}
		\int_{\mathcal Z}
		\ell\!\left(\mathbf y,f(\rho,\mathbf x)\right)
		\,d\rho(\mathbf x,\mathbf y)\,d\mathbb P(\rho).
	\end{equation}
	
	Let $\mathcal F$ be a class of measurable in-context predictors
	$f:\mathcal P(\mathcal Z)\times\mathcal X\to\mathcal Y$. In the subsequent
	sections, this class will be instantiated by predictors induced by Transformer
	architectures. The corresponding population benchmark is
	\begin{equation}
		\label{eq:global_error}
		\inf_{f\in\mathcal F} R_{\mathbb P}^{\ell}(f).
	\end{equation}

	\paragraph{Relation to standard learning formulations.}
	The expected risk in \eqref{eq:lifted_risk} differs from standard learning
	risks in how task information enters the prediction rule. In classical
	statistical learning \citep{vapnik2013nature}, a task distribution $\rho$ is
	fixed and one learns a task-specific predictor $h:\mathcal X\to\mathcal Y$,
	with risk
	\[
	R_{\rho}^{\ell}(h)
	:=
	\mathbb E_{(\mathbf X,\mathbf Y)\sim\rho}
	\ell\bigl(\mathbf Y,h(\mathbf X)\bigr)
	=
	\int_{\mathcal Z}
	\ell\!\left(\mathbf y,h(\mathbf x)\right)\,
	d\rho(\mathbf x,\mathbf y).
	\]
	In multi-task learning, one typically learns predictors of the form
	$h_t\circ g$, where $g$ is a shared representation and $h_t$ is a task-specific
	head \citep{maurer2016benefit}. For tasks $\rho_1,\ldots,\rho_T$, the average
	risk is
	\[
	R_{(\rho_t)_{t=1}^T}^{\ell}(\{h_t\circ g\}_{t=1}^T)
	:=
	\frac{1}{T}\sum_{t=1}^T
	\int_{\mathcal Z}
	\ell\!\left(\mathbf y,(h_t\circ g)(\mathbf x)\right)\,
	d\rho_t(\mathbf x,\mathbf y).
	\]
	Distribution regression instead learns a map
	$r:\mathcal P(\mathcal X)\to\mathcal Y$ from distribution-label pairs
	$(P_{\mathbf X},\mathbf Y)\sim\Pi
	\in\mathcal P(\mathcal P(\mathcal X)\times\mathcal Y)$
	\citep{poczos2013distribution,szabo2016learning}, with risk
	\[
	R_{\Pi}^{\ell}(r)
	:=
	\int_{\mathcal P(\mathcal X)\times\mathcal Y}
	\ell\!\left(\mathbf y,r(P_{\mathbf X})\right)\,
	d\Pi(P_{\mathbf X},\mathbf y).
	\]
	By contrast, in the ICL formulation above, the same predictor
	$f:\mathcal P(\mathcal Z)\times\mathcal X\to\mathcal Y$ is used across tasks.
	Task information enters as a distributional argument of $f$, and prediction for
	a query input is represented by $f(\rho,\mathbf x)$ at the population level or
	by $f(\widehat\rho^{\,k},\mathbf x)$ with a finite context.

	\section{Pretraining Objectives for In-Context Learning}
	\label{pretrain}
	
	We now instantiate the preceding statistical formulation through two idealized
	pretraining objectives, which serve as theoretical simplifications of the
	prediction structures in causal and masked language models. For causal language
	models, such as GPT-style models
	\citep{radford2019language,brown2020language}, we use the term autoregressive
	objective to emphasize the left-to-right prediction structure. For masked
	language models, such as BERT and DeBERTa \citep{devlin2019bert,hedeberta}, we
	use the term masked objective to emphasize prediction from bidirectional context
	with the target removed. In our formulation, the distinction between the
	autoregressive and masked objectives is encoded by the empirical measure used as
	the  context for each target prediction.

	Let \((\rho_t)_{t=1}^T\sim\mathbb P^{\otimes T}\) be i.i.d.\ task distributions.
	For each task \(t\), let \(\mathbf z_{t,j}=(\mathbf x_{t,j},\mathbf y_{t,j})\),
	\(j\in[N]\), be i.i.d.\ samples drawn from \(\rho_t\). We assume \(N\ge 2\).
	
	\paragraph{Masked objective.}
	For the masked objective, the target example is removed from the context. For
	each \(j\in[N]\), define the leave-one-out empirical measure
	\[
	\widehat\rho_t^{\,(-j)}
	:=
	\frac{1}{N-1}
	\sum_{j'\neq j}\delta_{\mathbf z_{t,j'}} .
	\]
	When predicting \(\mathbf y_{t,j}\), the model receives the query input
	\(\mathbf x_{t,j}\) and the context represented by
	\(\widehat\rho_t^{\,(-j)}\). The corresponding empirical risk is
	\begin{equation}
		\label{eq:lifted_empirical_risk}
		\widehat R_{T,N}^{\ell}(f)
		:=
		\frac{1}{TN}
		\sum_{t=1}^T
		\sum_{j=1}^N
		\ell\!\left(
		\mathbf y_{t,j},
		f(\widehat\rho_t^{\,(-j)},\mathbf x_{t,j})
		\right).
	\end{equation}
	This construction abstracts the key statistical feature of masked prediction:
	the target label is excluded from the conditioning context. It is not meant to
	replicate the exact masking scheme used in practical masked language modeling,
	where masking is typically random and partial; rather, it provides a clean
	idealization for risk analysis \citep{devlin2019bert,salazar2020masked}.
	
	\paragraph{Autoregressive objective.}
    For the autoregressive objective, the target response is predicted from preceding
    examples only. For \(j\ge 2\), define the prefix empirical measure
    \[
    	\widehat\rho_t^{(<j)}
    	:=
    	\frac{1}{j-1}
    	\sum_{j'=1}^{j-1}\delta_{\mathbf z_{t,j'}} .
    \]
    The corresponding empirical risk is
    \begin{equation}
    	\label{eq:lifted_empirical_risk_bar}
    	\overline R_{T,N}^{\ell}(f)
    	:=
    	\frac{1}{T(N-1)}
    	\sum_{t=1}^T
    	\sum_{j=2}^N
    	\ell\!\left(
    	\mathbf y_{t,j},
    	f(\widehat\rho_t^{(<j)},\mathbf x_{t,j})
    	\right).
    \end{equation}
    The term \(j=1\) is omitted because no preceding example is available. This
    objective is the ICL analogue of autoregressive next-token prediction: the
    target response is predicted from the preceding context
    \citep{radford2019language,brown2020language}.
	
	Thus, within the same predictor class
	\(f:\mathcal P(\mathcal Z)\times\mathcal X\to\mathcal Y\), the masked and
	autoregressive objectives differ only in the context measure used for each
	prediction: \(\widehat\rho_t^{\,(-j)}\) for the masked objective and
	\(\widehat\rho_t^{(<j)}\) for the autoregressive objective.

    \section{Transformer Architecture}
	\label{sec:transformer_arch}
	
	We now specialize the measure-dependent learning framework to Transformer
	architectures. In ICL, the input consists of context examples and a query input.
	We encode the context by an empirical measure, while the query input is kept as
	the explicit input variable. This leads to a Transformer predictor of the form
	\(f(\rho,\mathbf x)\), where \(\rho\) represents the context and
	\(\mathbf x\) is the query input.
	
	For theoretical clarity, we omit positional encodings. The resulting layers are
	permutation-equivariant with respect to the input tokens, so a context sequence
	can be represented by the empirical measure induced by its tokens. Recall that
	\(\mathcal Z=\mathcal X\times\mathcal Y\subseteq\mathbb R^{d_x+d_y}\), and write
	\(d_{\mathcal Z}:=d_x+d_y\). In this section, we idealize the embedding step by
	treating each pair token as already represented in \(\mathcal Z\). We assume
	that \(\mathcal Z\subseteq\{\mathbf u\in\mathbb R^{d_{\mathcal Z}}:
	\|\mathbf u\|_2\le R\}\) is compact. For an input sequence
	\(\mathbf Z=(\mathbf z_1,\ldots,\mathbf z_L)\in\mathcal Z^L\), write
	\(\widehat\rho^{\,L}(\mathbf Z):=L^{-1}\sum_{\ell=1}^L
	\delta_{\mathbf z_\ell}\).
	
	In this formulation, self-attention is viewed as a token-wise map whose
	coefficients depend on \(\widehat\rho^{\,L}(\mathbf Z)\). Applying the map to
	all tokens transforms the point cloud and pushes forward its empirical measure.
	This measure-theoretic viewpoint provides a convenient way to study Transformer
	regularity through Wasserstein continuity arguments and is in line with related
	formulations of attention; see
	\citet{vuckovic2021regularity,sander2022sinkformers,mroueh2023towards,kawata2026transformers,furuya2026approximation}.
	
	\paragraph{Standing regularity conditions.}
	We adopt the measure-theoretic attention setting of
	\citet{vuckovic2021regularity}. Throughout this section, \(W_1\) denotes the
	\(1\)-Wasserstein distance induced by the Euclidean metric on \(\mathcal Z\).
	All attention and feedforward parameters are fixed and have finite operator
	norms. For each layer \(j\in[D]\) and head \(h\in[H]\), set
	\(G_j^{(h)}(\mathbf z,\mathbf u):=\exp(a(W_{Q,j}^{(h)}\mathbf z,
	W_{K,j}^{(h)}\mathbf u))\), where a standard choice is the scaled dot-product
	similarity \(a(\mathbf q,\mathbf k)=\mathbf q^\top\mathbf k/\sqrt{d_{\mathrm{att}}}\).
	Here \(d_{\mathrm{att}}\) denotes the query-key dimension. We assume
	\(G_j^{(h)}(\mathbf z,\mathbf u)\ge\varepsilon_{j,h}>0\) on
	\(\mathcal Z\times\mathcal Z\), and \(G_j^{(h)}\) is Lipschitz in each argument
	uniformly over the other, namely
	\(\sup_{\mathbf u\in\mathcal Z}\operatorname{Lip}(G_j^{(h)}(\cdot,\mathbf u))<\infty\)
	and
	\(\sup_{\mathbf z\in\mathcal Z}\operatorname{Lip}(G_j^{(h)}(\mathbf z,\cdot))<\infty\).
	These conditions hold for scaled dot-product attention on compact domains with
	finite parameter norms. We also assume that \(\mathcal Z\) is invariant under
	the Transformer blocks below, so all intermediate token representations remain
	in \(\mathcal Z\).
	
	\paragraph{Self-attention as a measure-dependent map.}
	For any \(\rho\in\mathcal P(\mathcal Z)\), define the \(h\)-th attention head in
	layer \(j\) by
	\[
	\mathcal A_{j,\rho}^{(h)}(\mathbf z)
	:=
	\frac{
		\int_{\mathcal Z}
		G_j^{(h)}(\mathbf z,\mathbf u)W_{V,j}^{(h)}\mathbf u
		\,d\rho(\mathbf u)
	}{
		\int_{\mathcal Z}
		G_j^{(h)}(\mathbf z,\mathbf u)
		\,d\rho(\mathbf u)
	},
	\qquad
	\mathbf z\in\mathcal Z .
	\]
	The denominator is positive by the lower bound on \(G_j^{(h)}\). The multi-head
	attention map is
	\(\mathcal A_{j,\rho}(\mathbf z):=\sum_{h=1}^H
	W_{O,j}^{(h)}\mathcal A_{j,\rho}^{(h)}(\mathbf z)\). If
	\(\rho=\widehat\rho^{\,L}(\mathbf Z)\), then this integral formula reduces to
	the usual finite-token softmax attention:
	\[
	\mathcal A_{j,\widehat\rho^{\,L}(\mathbf Z)}^{(h)}(\mathbf z_\ell)
	=
	\frac{
		\sum_{s=1}^L
		G_j^{(h)}(\mathbf z_\ell,\mathbf z_s)
		W_{V,j}^{(h)}\mathbf z_s
	}{
		\sum_{s=1}^L
		G_j^{(h)}(\mathbf z_\ell,\mathbf z_s)
	},
	\qquad
	\ell\in[L].
	\]
	Indeed, since \(G_j^{(h)}(\mathbf z_\ell,\mathbf z_s)
	=\exp(a(W_{Q,j}^{(h)}\mathbf z_\ell,W_{K,j}^{(h)}\mathbf z_s))\), this is the
	standard softmax attention formula.
	
	\paragraph{Residual block and empirical pushforward.}
	Let \(\sigma\) be the ReLU function applied componentwise, and define the
	feedforward map in layer \(j\) by
	\(g_j(\mathbf z):=W_j^{(2)}\sigma(W_j^{(1)}\mathbf z+\mathbf b_j^{(1)})
	+\mathbf b_j^{(2)}\). Since ReLU is \(1\)-Lipschitz, \(g_j\) is Lipschitz with
	constant at most \(\|W_j^{(2)}\|_{\mathrm{op}}\|W_j^{(1)}\|_{\mathrm{op}}\).
	For a measure \(\mu\in\mathcal P(\mathcal Z)\), define the layer-\(j\) token
	block by \(B_j(\mu,\mathbf z):=\mathbf z+\mathcal A_{j,\mu}(\mathbf z)
	+g_j(\mathbf z+\mathcal A_{j,\mu}(\mathbf z))\). Thus
	\(B_j(\mu,\cdot)\) is the residual attention-feedforward block applied to a
	token, with attention coefficients determined by \(\mu\).
	
	We define the depth-\(D\) token map induced by an initial measure \(\mu\) as
	follows. Set \(\mu_0:=\mu\) and \(\mathbf z^{(0)}:=\mathbf z\). For
	\(j=1,\ldots,D\), let
	\(\mathbf z^{(j)}:=B_j(\mu_{j-1},\mathbf z^{(j-1)})\) and
	\(\mu_j:=(B_j(\mu_{j-1},\cdot))_\#\mu_{j-1}\). We write
	\(\mathcal T_\mu^{(D)}(\mathbf z):=\mathbf z^{(D)}\). Hence the final measure
	generated by the same layer recursion is \(\mu_D=(\mathcal T_\mu^{(D)})_\#\mu\).
	
	For an empirical input sequence
	\(\mathbf Z=(\mathbf z_1,\ldots,\mathbf z_L)\), take
	\(\mu=\widehat\rho^{\,L}(\mathbf Z)\). Applying the above recursion to each
	initial token gives
	\(\mathbf z_\ell^{(j)}:=B_j(\mu_{j-1},\mathbf z_\ell^{(j-1)})\) and
	\(\mu_j=L^{-1}\sum_{\ell=1}^L\delta_{\mathbf z_\ell^{(j)}}\) for
	\(j=1,\ldots,D\). Therefore the finite-token Transformer output is
	\(\mathcal T^{(D)}(\mathbf Z)=
	(\mathcal T_\mu^{(D)}(\mathbf z_1),\ldots,
	\mathcal T_\mu^{(D)}(\mathbf z_L))\), where
	\(\mu=\widehat\rho^{\,L}(\mathbf Z)\).
	
	\paragraph{Transformer predictors for ICL.}
	For ICL, we use pair tokens in \(\mathcal Z\). Each observed context example is
	encoded as \(\mathbf z_i:=(\mathbf x_i,\mathbf y_i)\in\mathcal Z\). For a query
	input \(\mathbf x\), the corresponding label is unknown and is the quantity to
	be predicted. We therefore replace the missing label by a fixed placeholder
	\(\mathbf y_{\mathrm{mask}}\in\mathcal Y\), and write
	\(\widetilde{\mathbf z}(\mathbf x):=(\mathbf x,\mathbf y_{\mathrm{mask}})\).
	The placeholder is fixed independently of the true query label and therefore
	carries no information about that label.
	
	For any probability measure \(\rho\in\mathcal P(\mathcal Z)\), interpreted as a
	context measure, define the context-dependent query predictor by
	\[
	f(\rho,\mathbf x)
	:=
	W_{\mathrm{out}}
	\mathcal T_\rho^{(D)}
	\bigl(\widetilde{\mathbf z}(\mathbf x)\bigr),
	\]
	where \(W_{\mathrm{out}}:\mathbb R^{d_{\mathcal Z}}\to\mathbb R^{d_y}\) is
	linear. Given context examples \(\{(\mathbf x_i,\mathbf y_i)\}_{i=1}^k\), write
	\(\widehat\rho^{\,k}:=k^{-1}\sum_{i=1}^k
	\delta_{(\mathbf x_i,\mathbf y_i)}\). The resulting \(k\)-shot predictor is
	\(f(\widehat\rho^{\,k},\mathbf x)\). The measure argument represents the
	context, while the query token is updated by the Transformer token map induced
	by this context measure.
	
	\begin{theorem}[Transformer predictors are jointly Lipschitz]
		\label{thm:transformer_predictor_joint_lip}
		Suppose the standing regularity conditions in Section~\ref{sec:transformer_arch}
		hold. Then there exists a constant \(L_F>0\), depending only on the fixed
		Transformer parameters and the standing regularity constants, such that, for all
		\(\rho,\rho'\in\mathcal P(\mathcal Z)\) and
		\(\mathbf x,\mathbf x'\in\mathcal X\),
		\[
		\|f(\rho,\mathbf x)-f(\rho',\mathbf x')\|_2
		\le
		L_F\bigl(W_1(\rho,\rho')+\|\mathbf x-\mathbf x'\|_2\bigr).
		\]
	\end{theorem}
	
	The Wasserstein--Lipschitz bound in
	Theorem~\ref{thm:transformer_predictor_joint_lip} allows masked and
	autoregressive pretraining objectives to be treated in the same framework. Both
	objectives use the predictor \(f(\rho,\mathbf x)\) and differ only in the
	empirical measure representing the visible context: the masked objective uses
	the leave-one-out measure \(\widehat\rho_t^{\,(-j)}\), whereas the
	autoregressive objective uses the prefix measure \(\widehat\rho_t^{(<j)}\).
	A related Lipschitz-type result for Transformer predictors appears in
	\citet[Proposition~2]{mroueh2023towards}. The proof is given in
	Appendix~\ref{app:proof_transformer_predictor_joint_lip}.

	\section{Wasserstein Excess-Risk Bounds for Pretraining Objectives}
	To quantify discrepancies between task distributions,
	we measure distances in the Wasserstein metric.
	The $p$-Wasserstein distance between two probability measures
	$\rho,\rho'\in\mathcal P(\mathcal Z)$ is defined as follows.
	\begin{definition}[Wasserstein distance]
		\label{def:Wp_base}
		Let \(p\ge 1\). For \(\rho,\rho'\in\mathcal P(\mathcal Z)\), the
		\(p\)-Wasserstein distance induced by the Euclidean norm on
		\(\mathcal Z\subseteq\mathbb R^{d_x+d_y}\) is defined as
		\[
		W_p(\rho,\rho')
		:=
		\left(
		\inf_{\gamma\in\Gamma(\rho,\rho')}
		\int_{\mathcal Z\times\mathcal Z}
		\|\mathbf z-\mathbf z'\|_2^p
		\,d\gamma(\mathbf z,\mathbf z')
		\right)^{1/p},
		\]
		where \(\Gamma(\rho,\rho')\) denotes the set of all couplings between
		\(\rho\) and \(\rho'\).
	\end{definition}

	We study the transfer of a model trained under a meta-distribution
	$\mathbb P \in \mathcal P(\mathcal P(\mathcal Z))$
	to unseen tasks generated from another meta-distribution
	$\mathbb Q \in \mathcal P(\mathcal P(\mathcal Z))$.
	To compare the two meta-distributions,
	we introduce the following lifted Wasserstein distance
	on $\mathcal P(\mathcal P(\mathcal Z))$
	following \citet{carlier2024quantitative,mroueh2023towards}.
	\begin{definition}[Lifted Wasserstein distance]
		\label{def:lifted_W1}
		The lifted $1$-Wasserstein distance between
		$\mathbb P,\mathbb Q\in\mathcal P(\mathcal P(\mathcal Z))$
		is denoted by $\mathbb W_1(\mathbb P,\mathbb Q)$ and defined as
		\[
		\mathbb W_1(\mathbb P,\mathbb Q)
		:=
		\inf_{\pi\in\Gamma(\mathbb P,\mathbb Q)}
		\int_{\mathcal P(\mathcal Z)\times\mathcal P(\mathcal Z)}
		W_1(\rho,\rho')\, d\pi(\rho,\rho'),
		\]
		where $\Gamma(\mathbb{P},\mathbb{Q})$ denotes the set of all couplings between two meta-distributions $\mathbb{P}$ and $\mathbb{Q}$. 
	\end{definition}
	
	\begin{definition}[Wasserstein dimensions \citep{weed2019sharp}]
		\label{def:wasserstein_dimensions}
		Let $(\Omega,d)$ be a metric space. For $S\subset \Omega$, let
		$\mathcal N_\varepsilon(S)$ denote the smallest number of closed balls of
		diameter $\varepsilon$ needed to cover $S$. For a probability measure
		$\mu$ on $\Omega$ and $\tau\in[0,1)$, define
		$\mathcal N_\varepsilon(\mu,\tau):=\inf\{\mathcal N_\varepsilon(S):\mu(S)\ge 1-\tau\}$
		and
		$d_\varepsilon(\mu,\tau):=\log \mathcal N_\varepsilon(\mu,\tau)/\log(1/\varepsilon)$.
		For $p\ge 1$, the \emph{upper Wasserstein dimension} of $\mu$ is
		\[
		d_p^*(\mu):=\inf\Bigl\{s>2p:\limsup_{\varepsilon\to 0}
		d_\varepsilon\bigl(\mu,\varepsilon^{\frac{sp}{s-2p}}\bigr)\le s \Bigr\},
		\]
		and the \emph{lower Wasserstein dimension} of $\mu$ is
		$
		d_*(\mu):=\lim_{\tau \to 0}\,\liminf_{\varepsilon \to 0} d_\varepsilon(\mu,\tau).
		$
	\end{definition}

	When the target measure $\rho$ is supported on a compact metric space with
	diameter at most one, \citet{weed2019sharp} obtained
	$
	N^{-1/d_*(\rho)}
	\lesssim
	\mathbb E\bigl[W_p(\widehat\rho^{\,N},\rho)\bigr]
	\lesssim
	N^{-1/d_p^*(\rho)} .
	$
	Thus, Wasserstein dimensions identify the effective dimension governing
	empirical Wasserstein convergence in the compactly supported case.
	
	In this paper, the dimension condition is imposed on the meta-distribution
	$\mathbb P$, viewed as a probability measure on the metric task space
	$(\mathcal P(\mathcal Z),W_1)$. Although the compactness assumptions above
	would allow one to use the Wasserstein dimension in
	Definition~\ref{def:wasserstein_dimensions}, the following
	$(p,q)$-Wasserstein dimension provides a more general formulation, covering
	measures with possibly unbounded support under finite $q$-moment conditions.

	\begin{definition}[$(p,q)$-Wasserstein dimension {\citep{chakraborty2026generalization}}]
		\label{def:pq_wass_dim}
		Let $(\Omega,d)$ be a metric space, and let $\mu$ be a probability measure on $\Omega$. For $0<p<q<\infty$, define
		\[
		d_{p,q}^*(\mu):=\inf\Bigl\{s>2p:\limsup_{\varepsilon\to 0}
		\frac{\log \mathcal N_\varepsilon\bigl(\mu,\varepsilon^{\frac{spq}{(q-p)(s-2p)}}\bigr)}
		{\log(1/\varepsilon)}\le s\Bigr\}.
		\]
	\end{definition}
	The relation between $d_{p,q}^*$ and other notions of intrinsic dimension is
	discussed in \citet[Proposition~9]{chakraborty2026generalization}. Since each
	task is represented by a distribution $\rho\in\mathcal P(\mathcal Z)$, the
	meta-distribution $\mathbb P$ is a distribution over the task space
	$\mathcal P(\mathcal Z)$. With this task space equipped with the Wasserstein
	metric, $d_{p,q}^*(\mathbb P)$ serves as the effective meta-level dimension
	governing empirical Wasserstein approximation from finitely many sampled tasks.

	\begin{assumption}[Lipschitz loss]
		\label{assum:loss_lipschitz}
		The loss function $\ell:\mathcal Y\times\mathcal Y\to\mathbb R$ is $L_\ell$-Lipschitz in both arguments, meaning that for all $\mathbf y,\mathbf y',\mathbf v,\mathbf v'\in\mathcal Y$,
		\[
		\bigl|\ell(\mathbf y,\mathbf v)-\ell(\mathbf y',\mathbf v')\bigr|
		\le
		L_\ell\bigl(\|\mathbf y-\mathbf y'\|_2 + \|\mathbf v-\mathbf v'\|_2\bigr).
		\]
	\end{assumption}
	
	This assumption is satisfied by many loss functions commonly used in machine learning \cite{ciampiconi2023survey}. For regression, canonical examples include the absolute loss, the Huber loss, and the squared loss when restricted to a compact output domain. More generally, any loss of the form $\ell(\mathbf y,\mathbf y') = \phi(\|\mathbf y-\mathbf y'\|_2)$ with Lipschitz $\phi:\mathbb R_+\to\mathbb R$ satisfies Assumption~\ref{assum:loss_lipschitz}. Since $\mathcal Y$ is compact, Assumption~\ref{assum:loss_lipschitz} also implies that $\ell$ is bounded on $\mathcal Y\times\mathcal Y$.
	
	To state the general excess-risk bounds, we impose a regularity
	condition on the predictor \(f\in\mathcal F\) with respect to both the
	measure argument representing the context and the query point.
	Theorem~\ref{thm:transformer_predictor_joint_lip} shows that the
	Transformer predictors introduced in
	Section~\ref{sec:transformer_arch} satisfy this condition.
	
	\begin{assumption}[Lipschitz predictor]
		\label{assum:predictor_lipschitz}
		Each predictor \(f\in\mathcal F\) is \(L_f\)-Lipschitz jointly in the
		context measure and the query point, in the sense that
		\[
		\left|f(\rho,\mathbf x)-f(\rho',\mathbf x')\right|_2
		\le
		L_f\left(
		W_1(\rho,\rho')+\|\mathbf x-\mathbf x'\|_2
		\right),
		\]
		for all \(\rho,\rho'\in\mathcal P(\mathcal Z)\) and
		\(\mathbf x,\mathbf x'\in\mathcal X\). Here \(W_1\) denotes the
		1-Wasserstein distance on \(\mathcal P(\mathcal Z)\) induced by the
		Euclidean norm on \(\mathcal Z\subseteq\mathbb R^{d_x+d_y}\).
	\end{assumption}
	
	This assumption is standard in the analysis of diffusion and
	measure-dependent stochastic systems, where Lipschitz continuity ensures
	stability with respect to both the state and the underlying measure
	\citep{funaki1984certain,leobacher2022well}. It is also consistent with
	Theorem~\ref{thm:transformer_predictor_joint_lip}, which establishes such
	a joint Lipschitz property for Transformer-based predictors.
	\begin{assumption}[$(1,q)$-Wasserstein dimension]
		\label{assum:wass_dim}
		Assume that the meta-distribution
		$\mathbb P\in\mathcal P(\mathcal P(\mathcal Z))$ satisfies
		$
		d_{1,q}^*(\mathbb P)\le s
		$
		for some $q>1$ and $s>2$.
	\end{assumption}
	This assumption controls the effective dimension of the meta-distribution over task distributions and leads to the task-level term $T^{-1/s}$ in the bounds below.

	Let $(\rho_t)_{t=1}^T \sim \mathbb P^{\otimes T}$ and, for each task $t$, let $(\mathbf z_{t,j})_{j=1}^N \sim \rho_t^{\otimes N}$. Define the empirical meta-distribution $\hat{\mathbb P}_T := T^{-1}\sum_{t=1}^T\delta_{\rho_t}$, the empirical task distributions $\hat\rho_t^{\,N} := N^{-1}\sum_{j=1}^N\delta_{\mathbf z_{t,j}}$, and the induced empirical distribution $\hat{\mathbb P}_{T,N} := T^{-1}\sum_{t=1}^T\delta_{\hat\rho_t^{\,N}}$. The next theorem provides the upper bound of the excess risk $\mathbb E[R_{\mathbb P}^\ell(\hat f)] -R_{\mathbb P}^\ell(f^*)$.

    \begin{theorem}[Excess risk decomposition for ICL]\label{thm:icl_decomposition}
		Suppose that Assumptions~\ref{assum:loss_lipschitz} and
		\ref{assum:predictor_lipschitz} hold, and let
		\(C_{\ell,f}:=L_\ell\big(L_f+\sqrt{2}\max\{1,L_f\}\big)\).
		Let \(f^*\in\arg\min_{f\in\mathcal F}R_{\mathbb P}^{\ell}(f)\).
		Then the following bounds hold.
		\emph{(i) Masked ICL.}
		If \(\hat f\in\arg\min_{f\in\mathcal F}\hat R_{T,N}^{\ell}(f)\), then
		\[
		\begin{aligned}
			\mathbb E[R_{\mathbb P}^{\ell}(\hat f)]
			-
			R_{\mathbb P}^{\ell}(f^*)
			\le\;&
			2C_{\ell,f}\,\mathbb E[\mathbb W_1(\mathbb P,\hat{\mathbb P}_T)]
			+
			2C_{\ell,f}\,\mathbb E[\mathbb W_1(\hat{\mathbb P}_T,\hat{\mathbb P}_{T,N})]  \\
			&+
			\frac{2L_\ell L_f}{TN}
			\sum_{t=1}^T\sum_{j=1}^N
			\mathbb E\!\left[
			W_1(\hat\rho_t^{\,N},\hat\rho_t^{\,(-j)})
			\right].
		\end{aligned}
		\]
		
		\emph{(ii) Autoregressive ICL.}
		If \(\bar f\in\arg\min_{f\in\mathcal F}\bar R_{T,N}^{\ell}(f)\), then
		\[
		\begin{aligned}
			\mathbb E[R_{\mathbb P}^{\ell}(\bar f)]
			-
			R_{\mathbb P}^{\ell}(f^*)
			\le\;&
			2C_{\ell,f}\,\mathbb E[\mathbb W_1(\mathbb P,\hat{\mathbb P}_T)]
			+
			2C_{\ell,f}\,\mathbb E[\mathbb W_1(\hat{\mathbb P}_T,\hat{\mathbb P}_{T,N})]  \\
			&+
			\frac{2L_\ell L_f}{T(N-1)}
			\sum_{t=1}^T\sum_{j=2}^N
			\mathbb E\!\left[
			W_1(\hat\rho_t^{\,N},\hat\rho_t^{(<j)})
			\right]
			+
			\frac{4M_\ell}{N}.
		\end{aligned}
		\]
		Here \(\mathbb W_1\) denotes the lifted Wasserstein distance, \(W_1\)
		denotes the Wasserstein distance on \(\mathcal Z\), and
		\(M_\ell:=\sup_{y,y'\in\mathcal Y}\ell(y,y')<\infty\) denotes the uniform
		bound on the loss.
	\end{theorem}
    Theorem~\ref{thm:icl_decomposition} gives a unified excess-risk
decomposition for masked and autoregressive ICL. In both cases, the first
term is the task-level generalization error over the meta-distribution
\(\mathbb P\), and the second term is the within-task sampling error caused
by replacing each task distribution \(\rho_t\) with its empirical version
\(\hat\rho_t^{\,N}\). The remaining terms account for the mismatch between
the full empirical task measure \(\hat\rho_t^{\,N}\) and the context
measure actually used by the training objective. In the masked case this
context measure is the leave-one-out empirical measure
\(\hat\rho_t^{\,(-j)}\), whereas in the autoregressive case it is the
prefix empirical measure \(\hat\rho_t^{(<j)}\), together with a finite-length
boundary term of order \(N^{-1}\). The proof is given in
Appendix~\ref{sec:proof_icl_decomposition}.

    \begin{corollary}[Excess risk rates for ICL]\label{cor:icl_rates}
    Suppose that the assumptions of Theorem~\ref{thm:icl_decomposition} hold.
    Assume further that Assumption~\ref{assum:wass_dim} holds for the
    meta-distribution \(\mathbb P\), so that
    \[
    \mathbb E\!\left[
    \mathbb W_1(\hat{\mathbb P}_T,\mathbb P)
    \right]
    \lesssim
    T^{-1/s}.
    \]
    Then both the masked ICL estimator \(\hat f\) and the autoregressive ICL
    estimator \(\bar f\) satisfy
    \[
    \mathbb E[R_{\mathbb P}^{\ell}(f_{T,N})]
    -
    R_{\mathbb P}^{\ell}(f^*)
    \lesssim
    T^{-1/s}
    +
    N^{-1/(d_x+d_y)}
    +
    N^{-1},
    \]
    where \(f_{T,N}\) denotes either \(\hat f\) or \(\bar f\). The implicit
    constant depends only on \(L_\ell\), \(L_f\), \(M_\ell\), the diameters of
    \((\mathcal X,d_{\mathcal X})\) and \((\mathcal Y,d_{\mathcal Y})\), and the
    constants in the Wasserstein convergence bounds.
    \end{corollary}

    Corollary~\ref{cor:icl_rates} turns the decomposition in
Theorem~\ref{thm:icl_decomposition} into explicit statistical rates. The
term \(T^{-1/s}\) is the task-level generalization error induced by
estimating the meta-distribution \(\mathbb P\) from \(T\) tasks, whereas
\(N^{-1/(d_x+d_y)}\) is the within-task empirical Wasserstein rate on
\(\mathcal Z=\mathcal X\times\mathcal Y\). The remaining \(N^{-1}\) term is
due to the reduced-context structure of the training objectives. Thus the excess risk relative to \(f^*\) vanishes as \(T,N\to\infty\). The proof is deferred to Appendix~\ref{sec:proof_icl_rates}.

    \begin{remark}[Connection to implicit Bayesian inference]
		Consider the expected risk $R_{\mathbb P}^{\ell}(f)$ defined in
		\eqref{eq:lifted_risk}. For the squared loss
		$\ell(y,\hat y)=\|y-\hat y\|_2^2$, the unconstrained population minimizer is
		\[
		f_{\mathrm{Bayes}}(\rho,x)
		=
		\mathbb E_{\rho}[Y\mid X=x],
		\qquad
		\rho_X\text{-a.s. for }\mathbb P\text{-a.e. }\rho .
		\]
		Thus, the optimal prediction rule is indexed by the task distribution
		$\rho$. In finite-context ICL, the model must infer this task information
		from the in-context examples in order to predict the query label. This is in
		line with the implicit Bayesian inference view of in-context learning
		\citep{xie2022an,muller2022transformers,panwar2024incontext}. If
		$f_{\mathrm{Bayes}}\in\mathcal F$, then the excess-risk consistency above
		implies convergence to the Bayes rule in risk.
	\end{remark}

	\section{Few-Shot In-Context Learning}
	\label{sec:few_shot_icl}
	
	Large language models are typically pretrained using prompts of a fixed maximum
	context length. In our statistical model, the pretraining objective in
	Section~\ref{pretrain} assumes that each task provides $N$ training examples,
	which corresponds to the maximum context size used during pretraining. At
	inference time, however, the number of available in-context examples is usually
	smaller than the pretraining context length. We therefore consider the few-shot
	setting in which only $k\le N$ in-context examples are observed. The learner
	must then predict the label of a new query input using the empirical context
	formed by these $k$ examples.
	
	Suppose that a task distribution $\rho\sim\mathbb P$ is sampled from the
	meta-distribution at inference time. The learner observes a prompt
	$P^{(k)} := (\mathbf z_1,\ldots,\mathbf z_k,\mathbf x_{k+1})$, where
	$\mathbf z_j=(\mathbf x_j,\mathbf y_j)$ and
	$(\mathbf z_1,\ldots,\mathbf z_{k+1})\sim\rho^{\otimes(k+1)}$. The first $k$
	labeled pairs constitute the in-context examples, while $\mathbf x_{k+1}$ is the
	query input. Recall that we associate to the prompt the empirical context
	measure $\widehat\rho^{\,k}:=k^{-1}\sum_{j=1}^k\delta_{\mathbf z_j}$, and define
	the predictor output as $f(\widehat\rho^{\,k},\mathbf x_{k+1})$. The performance
	of a predictor $f\in\mathcal F$ is measured by the $k$-shot expected risk
	\begin{align}
		\label{eq:kshot_expected_risk}
		R^{\ell}_{k,\mathbb P}(f)
		&=
		\mathbb E_{\rho\sim\mathbb P}
		\mathbb E_{(\mathbf z,\mathbf z_1,\ldots,\mathbf z_k)\sim\rho^{\otimes(k+1)}}
		\Big[
		\ell\big(\mathbf y, f(\widehat\rho^{\,k}, \mathbf x)\big)
		\Big] \nonumber\\
		&=
		\int_{\mathcal P(\mathcal Z)}
		\int_{\mathcal Z}
		\ell\big(\mathbf y, f(\widehat\rho^{\,k}, \mathbf x)\big)
		\,\mathrm d\rho(\mathbf x,\mathbf y)
		\prod_{j=1}^k \mathrm d\rho(\mathbf x_j,\mathbf y_j)\,
		\mathrm d\mathbb P(\rho).
	\end{align}

	\begin{theorem}[Unified $k$-shot excess risk bound for ICL]
		\label{thm:kshot_excess_risk}
		Suppose Assumptions~\ref{assum:loss_lipschitz},
		\ref{assum:predictor_lipschitz}, and
		\ref{assum:wass_dim} hold.
		Consider $T$ training tasks with $N$ samples per task, and let
		$f_{T,N}\in\mathcal F$ denote an empirical risk minimizer associated
		with the training objective under consideration, namely, the masked
		objective or the autoregressive objective.
		Let $f_k^* \in \arg\min_{f\in\mathcal F} R_{k,\mathbb P}^{\ell}(f)$.
		Then
		\[
		\mathbb E\!\left[
		R_{k,\mathbb P}^{\ell}(f_{T,N})
		-
		R_{k,\mathbb P}^{\ell}(f_k^*)
		\right]
		\;\lesssim\;
		T^{-1/s}
		+
		N^{-1/(d_x+d_y)}
		+
		k^{-1/(d_x+d_y)}
		+
		N^{-1}.
		\]
		Moreover, if $\mathbb P=\delta_\rho$ for some
		$\rho\in\mathcal P(\mathcal Z)$, then the meta-level term is zero.
	\end{theorem}
	
	Theorem~\ref{thm:kshot_excess_risk} gives a unified excess-risk bound for both
	masked and autoregressive in-context learning. It studies the $k$-shot expected
	risk of the predictor $f_{T,N}$ obtained from pretraining on $T$ tasks with $N$
	samples per task. Here $N$ is the maximum context length used during
	pretraining, whereas the inference-time prompt contains only $k\le N$
	in-context examples. Thus the theorem separates the pretraining context length
	from the inference-time context size and shows how $T$, $N$, and $k$ enter the
	same excess-risk guarantee.
	
	This complements recent theoretical studies of task scaling and context scaling
	in in-context learning
	\citep{jiao2026beyond,zhang2024trained,bai2023transformers,li2023transformers,kim2024transformers,wakayama2025context,abedsoltan2024context}.
	These works typically focus on task scaling or context scaling, but do not
	explicitly separate the maximum context length used during pretraining from the
	actual context size used at inference time. This distinction is important
	because language models have a finite context capacity: $N$ is the maximum
	within-task context budget, while $k\le N$ is the number of examples actually
	provided in the inference-time prompt.
	
	A key feature of Theorem~\ref{thm:kshot_excess_risk} is that the training
	objective and the evaluation risk are different: the predictor is obtained from
	a masked or autoregressive pretraining objective, but is evaluated by the
	$k$-shot expected risk $R_{k,\mathbb P}^{\ell}$. The result also shows that
	masked and autoregressive objectives have the same upper-bound rate under the
	present statistical model. This is consistent with the empirical results in
	Section~\ref{sec:empirical}, which show that masked label prediction can induce
	in-context learning behavior on the synthetic tasks. The proof of
	Theorem~\ref{thm:kshot_excess_risk} is deferred to Appendix~\ref{kshot}.

	\begin{remark}[Role of the exponent $1/s$]
		The term $T^{-1/s}$ in Theorem~\ref{thm:kshot_excess_risk} is the
		task-level error arising from approximating the meta-distribution
		$\mathbb P$ using $T$ training tasks. The parameter $s$ measures the
		effective meta-level complexity of $\mathbb P$ under
		Assumption~\ref{assum:wass_dim}: a larger $s$ corresponds to a more complex
		meta-distribution and leads to a slower rate in $T$. This provides a
		theoretical way to quantify why the number and diversity of pretraining
		tasks can affect in-context learning, a phenomenon also emphasized in recent
		empirical and theoretical studies
		\citep{raventos2023pretraining,wumany,liu2025context}.
	\end{remark}

    Theorem~\ref{thm:kshot_excess_risk} also yields a data-allocation principle
for ICL pretraining. Under a fixed budget \(TN=B\), the terms \(T^{-1/s}\)
and \(N^{-1/(d_x+d_y)}\) capture the trade-off between task
diversity and within-task sample size. This trade-off is closely related to
data allocation in large-scale pretraining and meta-learning
\citep{wang2023data,hoffmann2022training,cioba2022distribute}. The following
result identifies the allocation of \(T\) and \(N\) that balances these two
sources of error.

	\begin{corollary}[Optimal allocation under a fixed data budget]
		\label{cor:budget_optimal_allocation}
		Suppose that the assumptions of Theorem~\ref{thm:kshot_excess_risk} hold and the total pretraining data budget satisfies $TN=B$. Then
		\[
		\inf_{T,N:\,TN=B}
		\mathbb E\!\left[
		R_{k,\mathbb P}^{\ell}(f_{T,N})
		-
		R_{k,\mathbb P}^{\ell}(f_k^*)
		\right]
		\lesssim
		B^{-1/(s+d_x+d_y)}
		+
		k^{-1/(d_x+d_y)}.
		\]
		Moreover, an order-optimal allocation is given by
		$T^\star \asymp B^{s/(s+d_x+d_y)}$ and
		$N^\star \asymp B^{(d_x+d_y)/(s+d_x+d_y)}$.
	\end{corollary}

	Corollary~\ref{cor:budget_optimal_allocation} shows that the optimal allocation
	balances the task-level term \(T^{-1/s}\) and the within-task term
	\(N^{-1/(d_x+d_y)}\). This gives \(T\asymp B^{s/(s+d_x+d_y)}\) and
	\(N\asymp B^{(d_x+d_y)/(s+d_x+d_y)}\), yielding the pretraining-dependent rate
	\(B^{-1/(s+d_x+d_y)}\). The remaining term \(k^{-1/(d_x+d_y)}\) is an inference-time error
	and cannot be reduced by increasing the pretraining budget alone. The proof is
	deferred to Appendix~\ref{sec:proof_budget_optimal_allocation}.

	Existing theoretical analyses of in-context learning often assume that the
	pretraining tasks and the inference task are drawn from the same
	meta-distribution. More precisely, each task is represented by a distribution
	$\rho\in\mathcal P(\mathcal Z)$, the pretraining tasks are sampled from
	$\mathbb P\in\mathcal P(\mathcal P(\mathcal Z))$, and the inference task is a
	new task also sampled from the same $\mathbb P$. Thus the inference task is new,
	but it is not drawn from a different meta-distribution.
	
	Recent work suggests that this distinction is important for understanding the
	OOD behavior of in-context learning. Empirically, ICL may fail to learn genuinely
	new input--output mappings beyond the pretraining task family, and may instead
	select mechanisms already supported by pretraining \citep{wang2025can}. From a
	geometric perspective, apparent OOD success may occur when the downstream task
	is compatible with task structures encountered during pretraining, whereas
	shifts outside this structure can lead to non-negligible risk
	\citep{kwon2026outofdistribution}. Complementarily, increasing task diversity
	can induce a transition from task-specialized to general-purpose ICL
	\citep{goddard2025can}. These findings motivate studying the case where the
	pretraining and inference tasks are drawn from possibly different
	meta-distributions.

	Accordingly, let $\mathbb P,\mathbb Q\in\mathcal P(\mathcal P(\mathcal Z))$
	denote the pretraining and inference meta-distributions, respectively. The case
	$\mathbb P=\mathbb Q$ reduces to the standard formulation, while
	$\mathbb P\ne\mathbb Q$ captures a shift in the distribution of tasks. The
	following theorem quantifies how this shift affects the expected risk through
	the lifted Wasserstein distance $\mathbb W_1(\mathbb P,\mathbb Q)$. A related
	Wasserstein transferability bound appears in
	\citet[Theorem~1]{mroueh2023towards}.

	\begin{theorem}[Transferability of In-Context Learning]
		\label{thm:transferability}
		Let $\mathbb P, \mathbb Q \in \mathcal P(\mathcal P(\mathcal Z))$.
		Assume that the loss function $\ell$ satisfies
		Assumption~\ref{assum:loss_lipschitz} with constant $L_\ell$, and that the
		predictor $f\in\mathcal F$ satisfies
		Assumption~\ref{assum:predictor_lipschitz} with constant $L_f$. Then
		\[
		\big| R_{\mathbb P}^{\ell}(f)-R_{\mathbb Q}^{\ell}(f) \big|
		\le
		C_{\ell,f}\,\mathbb W_1(\mathbb P,\mathbb Q),
		\]
		where $C_{\ell,f}:=L_\ell\big(L_f+\sqrt{2}\max\{1,L_f\}\big)$ and
		$\mathbb W_1(\mathbb P,\mathbb Q)$ is the lifted $1$-Wasserstein distance
		given in Definition~\ref{def:lifted_W1}.
	\end{theorem}
	
	Theorem~\ref{thm:transferability} implies that the expected risk of $f$ on tasks
	drawn from $\mathbb Q$ is bounded by its risk on the pretraining
	meta-distribution $\mathbb P$ plus a term proportional to
	$\mathbb W_1(\mathbb P,\mathbb Q)$. Let
	$f^*\in\arg\min_{f\in\mathcal F}R_{\mathbb P}^{\ell}(f)$. It follows from
	Theorem~\ref{thm:transferability} that
	$R_{\mathbb Q}^{\ell}(f^*)\le R_{\mathbb P}^{\ell}(f^*)+
	C_{\ell,f}\,\mathbb W_1(\mathbb P,\mathbb Q)$. Hence performance degradation under
	task-distribution shift is linearly controlled by the lifted Wasserstein
	distance. The proof of Theorem~\ref{thm:transferability} is deferred to
	Appendix~\ref{sec:proof_transferability}.
	
	Theorem~\ref{thm:transferability} controls how the risk of a fixed predictor
	changes when the meta-distribution is changed from $\mathbb P$ to $\mathbb Q$.
	We now apply this transfer argument to the predictor $f_{T,N}$ learned from
	pretraining tasks drawn from $\mathbb P$, and evaluate its $k$-shot risk under
	the inference meta-distribution $\mathbb Q$. This yields the following
	few-shot bound under task-distribution shift.

	\begin{theorem}[Few-shot transferability under task-distribution shift]
		\label{thm:kshot_transferability_shift}
		Suppose Assumptions~\ref{assum:loss_lipschitz},
		\ref{assum:predictor_lipschitz}, and~\ref{assum:wass_dim} hold. Let
		$\mathbb P,\mathbb Q\in\mathcal P(\mathcal P(\mathcal Z))$ be the
		pretraining and inference meta-distributions, respectively. Consider $T$
		training tasks drawn from $\mathbb P$, each with $N$ samples, and let
		$f_{T,N}\in\mathcal F$ be an empirical risk minimizer associated with the
		training objective under consideration. Let
		$f_{k,\mathbb Q}^*\in\arg\min_{f\in\mathcal F}R_{k,\mathbb Q}^{\ell}(f)$.
		Then
		\[
		\mathbb E\!\left[
		R_{k,\mathbb Q}^{\ell}(f_{T,N})
		-
		R_{k,\mathbb Q}^{\ell}(f_{k,\mathbb Q}^*)
		\right]
		\;\lesssim\;
		T^{-1/s}
		+
		N^{-1/(d_x+d_y)}
		+
		k^{-1/(d_x+d_y)}
		+
		N^{-1}
		+
		2C_{\ell,f}\,\mathbb W_1(\mathbb P,\mathbb Q),
		\]
		where \(C_{\ell,f}:=L_\ell\big(L_f+\sqrt2\max\{1,L_f\}\big)\).
	\end{theorem}
	
	Theorem~\ref{thm:kshot_transferability_shift} extends the \(k\)-shot
	excess-risk analysis to task-distribution shift. The predictor \(f_{T,N}\) is
	trained on tasks drawn from the pretraining meta-distribution \(\mathbb P\), but
	is evaluated on tasks drawn from the inference meta-distribution \(\mathbb Q\).
	This differs from the usual no-shift setting, where the pretraining tasks and
	the inference task are independent draws from the same meta-distribution.
	
	For example, consider a linear-regression task family with a fixed input
	distribution and a fixed noise distribution. Each task is determined by a
	regression coefficient \(\beta\), through the model
	\(Y=X^\top\beta+\varepsilon\). In the usual no-shift setting, the pretraining
	task parameters and the test-task parameter are sampled from the same prior, for
	instance \(\beta_1,\ldots,\beta_T\overset{\mathrm{i.i.d.}}{\sim}\Pi\) during
	pretraining and \(\beta_{\mathrm{test}}\sim\Pi\) at inference time. The inference
	task is therefore new, but it is not drawn from a different task-generating
	distribution. By contrast, task-distribution shift corresponds to the case where
	the pretraining task parameters are sampled from one prior and the test-task
	parameter from another, for instance
	\(\beta_1,\ldots,\beta_T\overset{\mathrm{i.i.d.}}{\sim}\Pi_{\mathbb P}\) and
	\(\beta_{\mathrm{test}}\sim\Pi_{\mathbb Q}\), with
	\(\Pi_{\mathbb P}\ne\Pi_{\mathbb Q}\). These two priors induce different
	meta-distributions \(\mathbb P\) and \(\mathbb Q\) over linear-regression tasks.
	
	The bound shows that allowing \(\mathbb Q\ne\mathbb P\) incurs an additional
	cost proportional to \(\mathbb W_1(\mathbb P,\mathbb Q)\), while the
	finite-sample terms in \(T\), \(N\), and \(k\) remain the same as in
	Theorem~\ref{thm:kshot_excess_risk}. Thus, if the inference meta-distribution is
	close to the pretraining meta-distribution in lifted Wasserstein distance, the
	few-shot excess risk remains controlled. To the best of our knowledge, prior ICL
	excess-risk analyses have not explicitly treated this \(k\)-shot setting with
	different pretraining and inference meta-distributions. When
	\(\mathbb P=\mathbb Q\), the shift term vanishes and the result reduces to the
	no-shift bound. The proof is deferred to
	Appendix~\ref{sec:proof_kshot_transferability_shift}.

    \section{Intrinsic Structure and the Curse of Dimensionality}
	\label{sec:intrinsic_structure}

	The ambient-dimensional dependence in
	Theorem~\ref{thm:kshot_excess_risk} comes from approximating each within-task
	distribution \(\rho\in\mathcal P(\mathcal Z)\) by empirical measures based on
	\(N\) or \(k\) samples. Since these measures are defined on
	\(\mathcal Z=\mathcal X\times\mathcal Y\), the empirical Wasserstein bounds
	yield the terms \(N^{-1/(d_x+d_y)}\) and \(k^{-1/(d_x+d_y)}\).

	In many high-dimensional applications, however, it is commonly postulated
	that the data are supported on, or concentrated near, a low-dimensional
	manifold \citep{narayanan2010sample,fefferman2016testing,popeintrinsic}.
	Such intrinsic low-dimensional structure can lead to faster convergence rates,
	since the relevant rates may depend on the intrinsic dimension rather than on
	the ambient dimension
	\citep{nakada2020adaptive,jiao2023deep,chakraborty2025statistical,chakraborty2026generalization}.

	In the present setting, this improvement affects only the within-task
	empirical Wasserstein approximation. Prior work shows that low-dimensional
	structure can reduce the effective dimension in empirical Wasserstein
	convergence \citep{canas2012learning,weed2019sharp}. Thus, under regular
	low-dimensional support, the terms involving $N$ and $k$ can be sharpened to
	intrinsic-dimensional rates. Under clusterable structure, they can be replaced
	by the rates $\sqrt{m/N}$ and $\sqrt{m/k}$. The task-level term $T^{-1/s}$
	remains unchanged.

	\begin{corollary}[Intrinsic-dimensional \(k\)-shot excess risk bounds]
		\label{cor:kshot_lowdim}
		Under the assumptions of Theorem~\ref{thm:kshot_excess_risk}, the bound can
		be sharpened under either of the following additional structural conditions.
		
		\medskip
		\noindent
		(a) If there exists a compact \(C^1\) submanifold
		\(\mathcal M\subseteq\mathcal Z\) with intrinsic dimension
		\(d_{\mathrm{int}}>2\) such that \(\operatorname{supp}(\rho)\subseteq
		\mathcal M\) for \(\mathbb P\)-a.e.\ \(\rho\), then
		\[
		\mathbb E\!\left[
		R^{\ell}_{k,\mathbb P}(f_{T,N})
		-
		R^{\ell}_{k,\mathbb P}(f_k^*)
		\right]
		\lesssim
		T^{-1/s}
		+
		N^{-1/d_{\mathrm{int}}}
		+
		k^{-1/d_{\mathrm{int}}}
		+
		N^{-1}.
		\]
		
		\medskip
		\noindent
		(b) Suppose \(\mathbb P\)-a.e.\ \(\rho\) is \((m,\Delta)\)-clusterable, and
		let \(D_{\mathcal Z}\ge \operatorname{diam}(\mathcal Z)\). If
		\(N,k\le m(D_{\mathcal Z}/(2\Delta))^2\), then
		\[
		\mathbb E\!\left[
		R^{\ell}_{k,\mathbb P}(f_{T,N})
		-
		R^{\ell}_{k,\mathbb P}(f_k^*)
		\right]
		\lesssim
		T^{-1/s}
		+
		D_{\mathcal Z}\sqrt{\frac{m}{N}}
		+
		D_{\mathcal Z}\sqrt{\frac{m}{k}}
		+
		N^{-1}.
		\]
	\end{corollary}

	Corollary~\ref{cor:kshot_lowdim} follows from the proof of
	Theorem~\ref{thm:kshot_excess_risk} by replacing the generic within-task
	empirical Wasserstein bounds. The manifold case uses
	Lemma~\ref{lem:wasserstein_concentration}, with \(d_{\mathrm{int}}>2\) giving
	the displayed rate; when the generic sample size is denoted by \(n\), the
	corresponding rates are \(n^{-1/2}\) for \(d_{\mathrm{int}}=1\) and
	\(n^{-1/2}\log n\) in the borderline case \(d_{\mathrm{int}}=2\). The
	clusterable case uses Lemma~\ref{thm:clusterable_wasserstein_D}, where
	\(D_{\mathcal Z}\) denotes a diameter bound for \(\mathcal Z\). Hence the curse
	of dimensionality is mitigated in the within-task terms, while the task-level
	term \(T^{-1/s}\) remains unchanged.

	\begin{remark}[Sharper manifold rates under stronger regularity]
		The manifold case in Corollary~\ref{cor:kshot_lowdim} uses only the compact
		\(C^1\) structure through Lemma~\ref{lem:wasserstein_concentration}. Under
		stronger geometric and distributional assumptions, such as finite reach and a
		density with respect to the volume measure on \(\mathcal M\), sharper
		empirical Wasserstein estimates from \citet{block2022intrinsic} can be used
		instead.
	\end{remark}

    \section{Empirical Verification}
	\label{sec:empirical}
	
	Recent work has used \emph{in-context function learning} as a controlled
	setting for studying how Transformers infer an unseen task from input--output
	examples. In particular, \citet{garg2022can} showed that GPT-2-style causal
	Transformers trained from scratch can learn simple function classes in context
	under synthetic regression protocols. Separately, \citet{samuel2024berts} showed
	that masked language models can exhibit generative in-context learning on
	natural-language tasks. Our theory gives same-order excess-risk upper bounds for the autoregressive and
	masked objectives. This motivates an empirical test of whether masked label
	prediction can also induce in-context function learning, beyond the
	autoregressive GPT-2-style setting.
	
	We study this question through a masked label-prediction protocol implemented by
	a bidirectional Transformer encoder operating on input--label pair tokens. We
	refer to this model as the Masked Pair Encoder. For each prompt, a latent
	function \(g\) is sampled from a task distribution. Given context examples
	\(D_n=\{(\mathbf x_i,\mathbf y_i)\}_{i=1}^n\), where
	\(\mathbf y_i=g(\mathbf x_i)\), and a query input \(\mathbf x_{n+1}\), the model
	predicts the missing label \(\mathbf y_{n+1}=g(\mathbf x_{n+1})\). During
	training, labels in the prompt are masked and the loss is computed only on
	masked positions. At evaluation time, all support labels are observed and only
	the query label is masked. Thus the model cannot access \(\mathbf y_{n+1}\) and
	must infer it from the in-context examples. We compare the Masked Pair Encoder
	against a GPT-2-style causal Transformer trained from scratch, following the
	synthetic function-class protocol of \citet{garg2022can}, and include
	task-specific baselines when applicable. Further details on the experimental setup are provided in
	Appendix~\ref{app:synthetic-icl}.
	
	\paragraph{Linear regression.}
	We first consider dense linear regression with \(d_{\mathrm{in}}=20\),
	\(w\sim \mathcal N(0,I_{d_{\mathrm{in}}})\),
	\(x_i\sim \mathcal N(0,I_{d_{\mathrm{in}}})\), and \(y_i=w^\top x_i\).
	This is the canonical setting where in-context prediction can be compared with
	ordinary least squares. As shown in Figure~\ref{fig:empirical}(a), the Masked
	Pair Encoder is nearly indistinguishable from the GPT-2 causal baseline and
	closely matches OLS once \(n>d_{\mathrm{in}}\), indicating that masked label
	prediction can implement in-context linear regression.
	
	\paragraph{Noisy linear regression.}
	We next evaluate the same noiselessly trained linear-regression checkpoints on
	noisy labels \(y_i^{\rm raw}=w^\top x_i+\epsilon_i\), where
	\(\epsilon_i\sim \mathcal N(0,1)\). The released noisy-linear curve uses
	population label renormalization and reports mean squared error divided by
	\(d_{\mathrm{in}}=20\). As shown in Figure~\ref{fig:empirical}(b), the Masked
	Pair Encoder remains close to the GPT-2 causal baseline and follows the
	least-squares trend away from the interpolation threshold. Near
	\(n=d_{\mathrm{in}}\), OLS exhibits the expected interpolation spike, consistent
	with double descent \citep{belkin2019reconciling}. This suggests that the
	Masked Pair Encoder learns a least-squares-like in-context estimator rather than
	merely memorizing noiseless prompts.
	
	\paragraph{Decision trees.}
	We then consider random depth-4 regression trees. Each internal node branches on
	the sign of a randomly selected input coordinate, and leaf values are drawn from
	a standard normal distribution. Figure~\ref{fig:empirical}(c) shows that the
	Masked Pair Encoder tracks the GPT-2 causal baseline and improves over
	nearest-neighbor, greedy-tree, and sign-preprocessed tree baselines in the
	large-context regime, suggesting that masked label prediction can support
	nonlinear in-context function learning.
	
	\paragraph{Two-layer ReLU networks.}
	Finally, we evaluate a neural function class. Each prompt samples a random
	two-layer ReLU network
	\[
	g(x)=\sqrt{2/h}\sum_{j=1}^h a_j\mathrm{ReLU}(u_j^\top x),
	\]
	with \(d_{\mathrm{in}}=20\), \(h=100\),
	\(u_j\sim \mathcal N(0,I_{d_{\mathrm{in}}})\), and
	\(a_j\sim \mathcal N(0,1)\). As shown in Figure~\ref{fig:empirical}(d), the
	Masked Pair Encoder closely tracks both the GPT-2 causal baseline and a
	per-prompt two-layer neural-network reference fit on the in-context examples.
	This provides a stronger nonlinear test because the task distribution itself is
	a randomly sampled neural network.
	
	\begin{figure}[t]
		\centering
		\begin{minipage}[t]{0.48\textwidth}
			\centering
			\includegraphics[width=\linewidth]{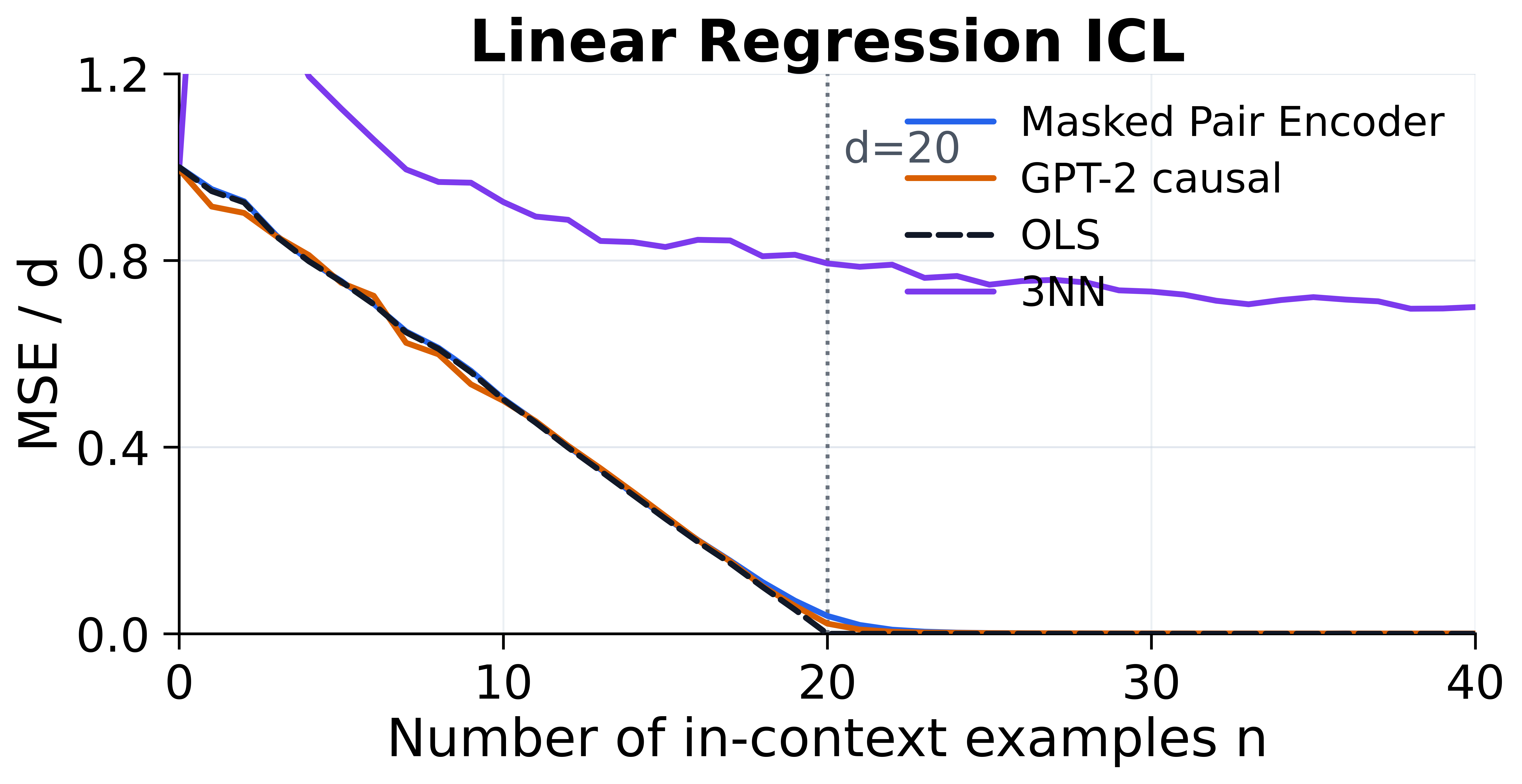}\\[-0.35em]
			{\small (a) Linear regression}
		\end{minipage}
		\hfill
		\begin{minipage}[t]{0.48\textwidth}
			\centering
			\includegraphics[width=\linewidth]{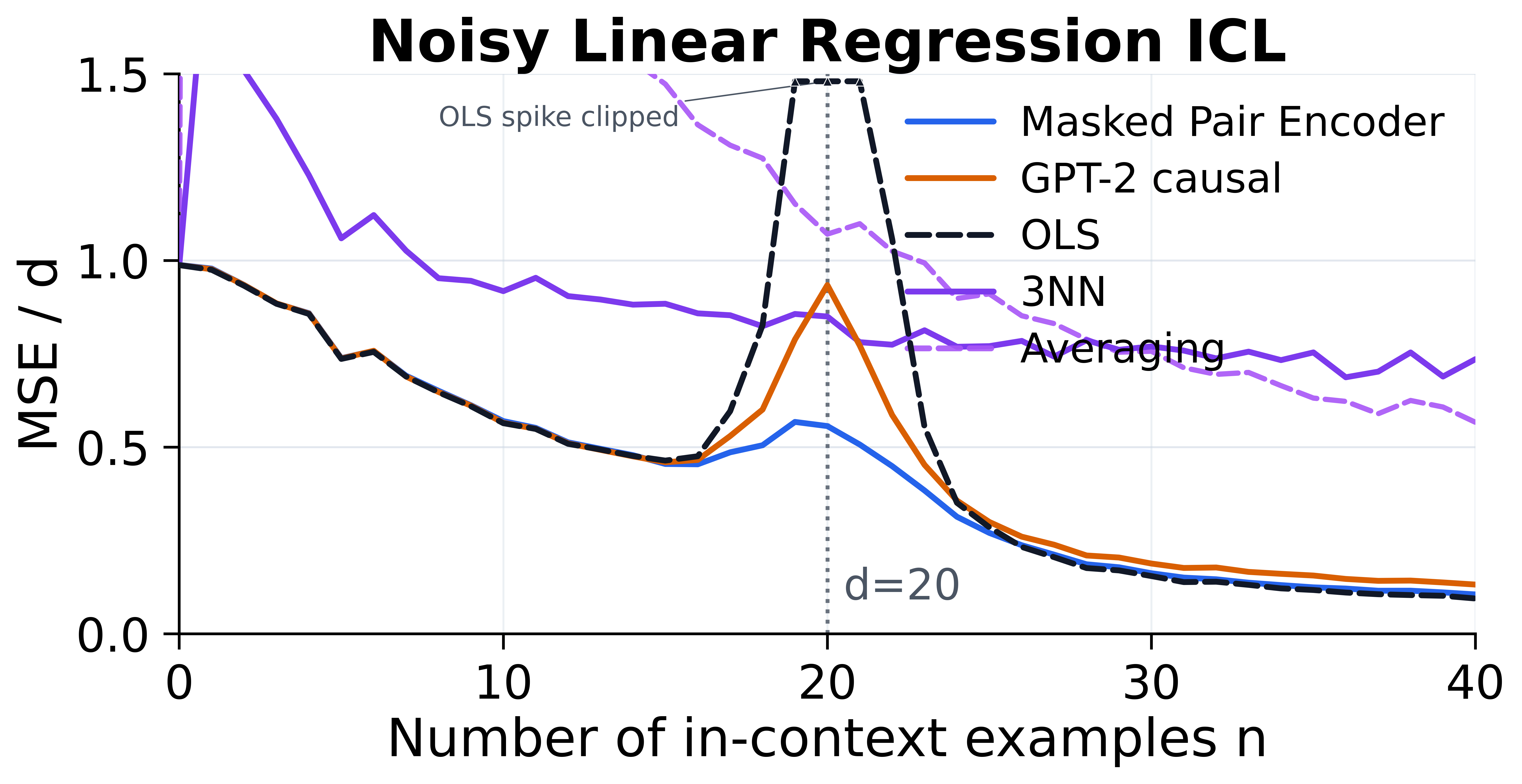}\\[-0.35em]
			{\small (b) Noisy linear regression}
		\end{minipage}
		
		\vspace{0.6em}
		
		\begin{minipage}[t]{0.48\textwidth}
			\centering
			\includegraphics[width=\linewidth]{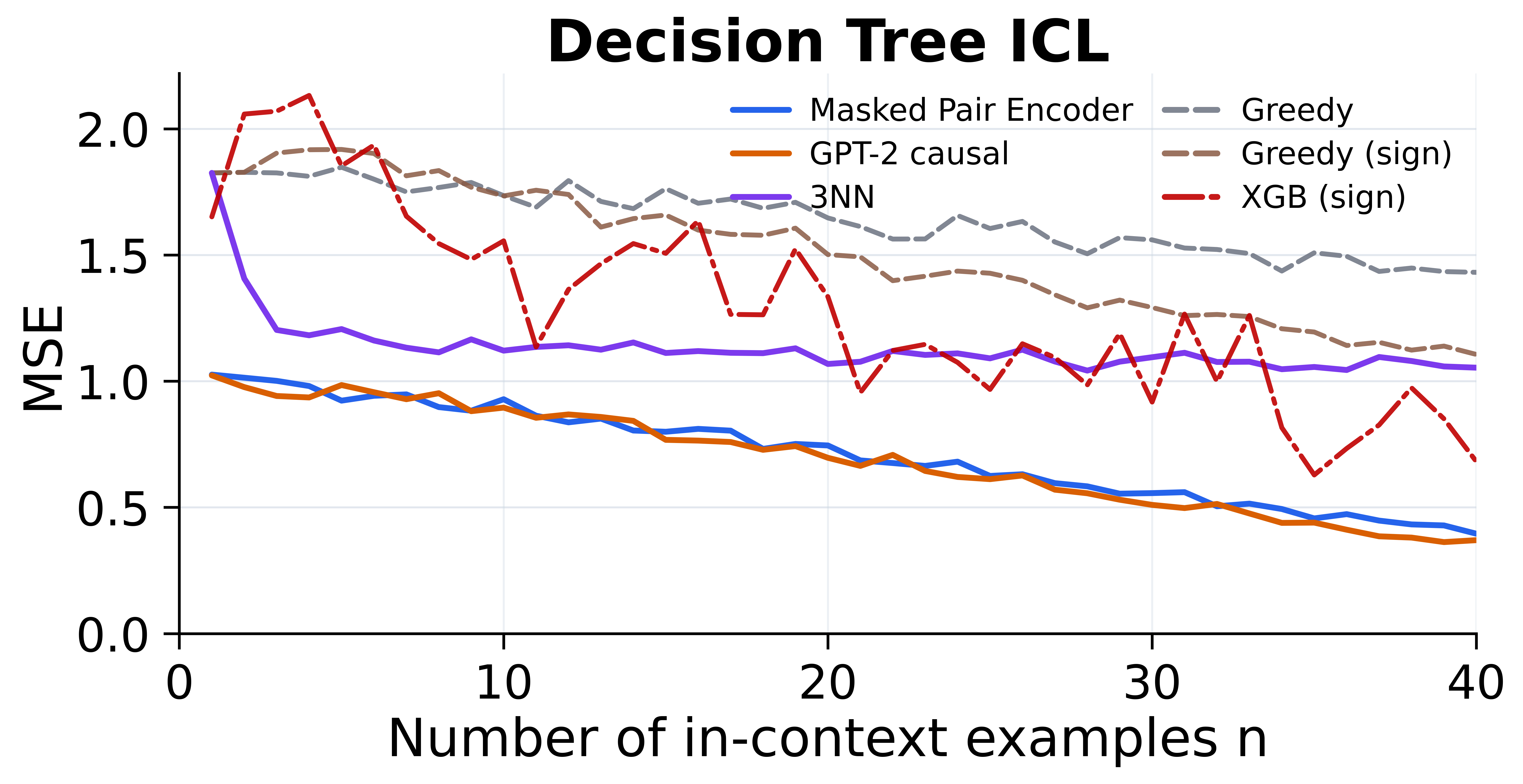}\\[-0.35em]
			{\small (c) Decision trees}
		\end{minipage}
		\hfill
		\begin{minipage}[t]{0.48\textwidth}
			\centering
			\includegraphics[width=\linewidth]{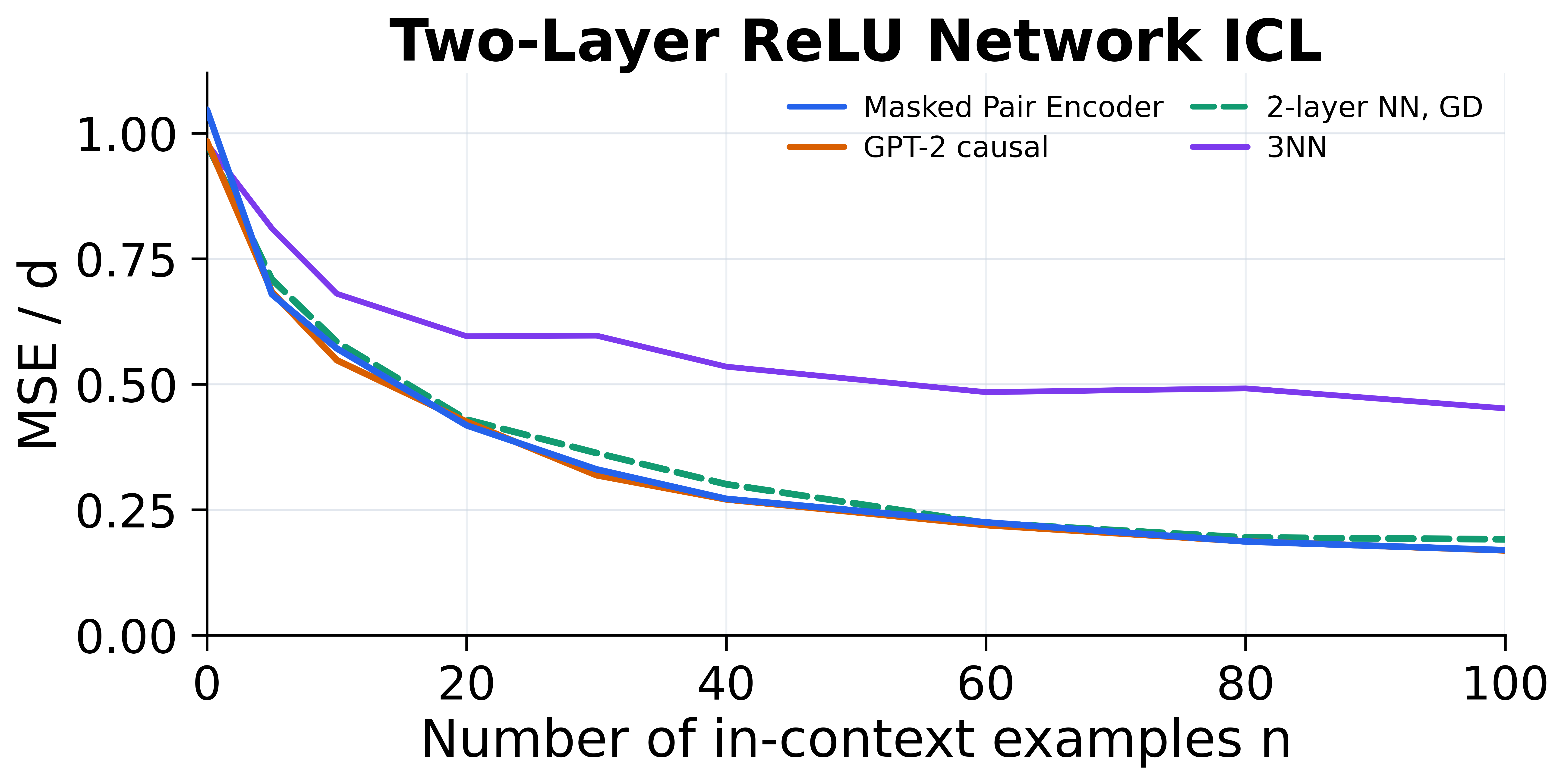}\\[-0.35em]
			{\small (d) Two-layer ReLU networks}
		\end{minipage}
		
		\vspace{0.3em}
		\caption{Synthetic in-context function learning.
			The Masked Pair Encoder is compared with a GPT-2-style causal Transformer
			trained from scratch and task-specific reference methods when available.
			Panels (a), (b), and (d) report mean squared error divided by the input
			dimension \(d_{\mathrm{in}}\), while panel (c) reports mean squared error.
			The four panels correspond to linear regression, noisy linear regression,
			depth-4 decision trees, and two-layer ReLU networks, respectively.}
		\label{fig:empirical}
	\end{figure}
	
	Together, these simulations provide evidence that masked label prediction can
	induce in-context function learning in controlled synthetic settings. The
	Masked Pair Encoder learns to exploit the observed input--output examples in the
	prompt to predict the masked query label, and its performance is comparable to
	that of a GPT-2-style causal Transformer across the four synthetic function
	classes. These results suggest that in-context function learning is not
	restricted to causal language models: masked language models can also acquire
	comparable in-context learning ability when trained with an appropriate masked objective.

	\section{Conclusion}
	
	This paper studied in-context learning by representing the context examples as
	an empirical measure and modeling prediction as $f(\rho,\mathbf x)$. Within
	this formulation, autoregressive and masked pretraining differ in how the
	context for each prediction is constructed: autoregressive prediction uses a
	prefix context, whereas masked prediction uses a leave-one-out context.
	
	The analysis shows that this difference does not change the main statistical
	rates under the proposed framework. For both objectives, the resulting bounds
	describe how the number of training tasks, the number of examples per task, and
	the number of in-context examples affect the final $k$-shot risk. The same
	framework also yields a data-allocation rule under a fixed pretraining budget,
	a transferability bound under task-distribution shift, and improved rates when
	the task distributions have low-dimensional or clusterable structure.
	
	The synthetic experiments further suggest that masked language models can also
	exhibit in-context learning behavior, indicating that ICL is not exclusive to
	causal language models.

	\textbf{Limitations.}
	This paper studies idealized autoregressive and masked objectives rather than
	full models of causal and masked language-model pretraining. The analysis also
	treats low-dimensional or clusterable structure as an explicit assumption,
	whereas modern neural networks may discover such structure adaptively during
	training. Finally, our experiments are limited to controlled synthetic
	function-learning tasks; broader evaluation on more diverse task families and
	natural-language benchmarks is left for future work.
	\newpage
	
	\appendix
	\section{Technical Lemmas and Proofs of Theoretical Results}
	\subsection{Proof of Theorem~\ref{thm:transformer_predictor_joint_lip}}
\label{app:proof_transformer_predictor_joint_lip}

We first record the attention stability estimate used in the proof.

\begin{lemma}[Attention stability]
	\label{lem:attention_stability}
	Under the standing regularity conditions in Section~\ref{sec:transformer_arch}, for each layer \(j\in[D]\) there exists \(L_{A,j}>0\) such that, for all \(\rho,\rho'\in\mathcal P(\mathcal Z)\) and \(\mathbf z,\mathbf z'\in\mathcal Z\),
	\[
	\|\mathcal A_{j,\rho}(\mathbf z)-\mathcal A_{j,\rho'}(\mathbf z')\|_2
	\le
	L_{A,j}\bigl(\|\mathbf z-\mathbf z'\|_2+W_1(\rho,\rho')\bigr).
	\]
\end{lemma}

\begin{proof}
	Fix \(j\in[D]\) and \(h\in[H]\). Let \(\Psi_{j,h}^{\rho}(\mathbf z)\) be the softmax probability measure associated with \(G_j^{(h)}(\mathbf z,\cdot)\) and \(\rho\), namely
	\(\Psi_{j,h}^{\rho}(\mathbf z)(d\mathbf u):=
	G_j^{(h)}(\mathbf z,\mathbf u)\rho(d\mathbf u)/
	\int_{\mathcal Z}G_j^{(h)}(\mathbf z,\mathbf v)\,d\rho(\mathbf v)\).
	Then \(\mathcal A_{j,\rho}^{(h)}(\mathbf z)=
	\int_{\mathcal Z}W_{V,j}^{(h)}\mathbf u\,
	\Psi_{j,h}^{\rho}(\mathbf z)(d\mathbf u)\). By Proposition~20 of
	\citet{vuckovic2021regularity}, applied with \(E=\mathcal Z\) and
	\(G=G_j^{(h)}\), there exist constants \(C_{j,h}^{z},C_{j,h}^{m}>0\) such that
	\(W_1(\Psi_{j,h}^{\rho}(\mathbf z),\Psi_{j,h}^{\rho}(\mathbf z'))
	\le C_{j,h}^{z}\|\mathbf z-\mathbf z'\|_2\) and
	\(W_1(\Psi_{j,h}^{\rho}(\mathbf z'),\Psi_{j,h}^{\rho'}(\mathbf z'))
	\le C_{j,h}^{m}W_1(\rho,\rho')\). Hence, with
	\(C_{j,h}:=\max\{C_{j,h}^{z},C_{j,h}^{m}\}\), the triangle inequality gives
	\(W_1(\Psi_{j,h}^{\rho}(\mathbf z),\Psi_{j,h}^{\rho'}(\mathbf z'))
	\le C_{j,h}(\|\mathbf z-\mathbf z'\|_2+W_1(\rho,\rho'))\).
	
	Let \(\pi\) be any coupling of \(\Psi_{j,h}^{\rho}(\mathbf z)\) and
	\(\Psi_{j,h}^{\rho'}(\mathbf z')\). Since these two measures are the marginals
	of \(\pi\),
	\[
	\mathcal A_{j,\rho}^{(h)}(\mathbf z)-\mathcal A_{j,\rho'}^{(h)}(\mathbf z')
	=
	\int_{\mathcal Z\times\mathcal Z}
	W_{V,j}^{(h)}(\mathbf u-\mathbf u')\,d\pi(\mathbf u,\mathbf u').
	\]
	Therefore
	\[
	\|\mathcal A_{j,\rho}^{(h)}(\mathbf z)-\mathcal A_{j,\rho'}^{(h)}(\mathbf z')\|_2
	\le
	\|W_{V,j}^{(h)}\|_{\mathrm{op}}C_{j,h}
	\bigl(\|\mathbf z-\mathbf z'\|_2+W_1(\rho,\rho')\bigr),
	\]
	where we have taken the infimum over all such couplings. Since
	\(\mathcal A_{j,\rho}(\mathbf z)=\sum_{h=1}^H
	W_{O,j}^{(h)}\mathcal A_{j,\rho}^{(h)}(\mathbf z)\), the desired bound follows
	with \(L_{A,j}:=\sum_{h=1}^H
	\|W_{O,j}^{(h)}\|_{\mathrm{op}}\|W_{V,j}^{(h)}\|_{\mathrm{op}}C_{j,h}\).
\end{proof}

\begin{lemma}[Layerwise Transformer block stability]
	\label{lem:block_stability}
	Under the assumptions of Lemma~\ref{lem:attention_stability}, for each layer
	\(j\in[D]\) there exist constants \(C_{z,j},C_{\rho,j},C_{p,j}>0\) such that, for all \(\rho,\rho'\in\mathcal P(\mathcal Z)\) and \(\mathbf z,\mathbf z'\in\mathcal Z\),
	\[
	\|B_j(\rho,\mathbf z)-B_j(\rho',\mathbf z')\|_2
	\le
	C_{z,j}\|\mathbf z-\mathbf z'\|_2+C_{\rho,j}W_1(\rho,\rho'),
	\]
	and
	\[
	W_1\bigl((B_j(\rho,\cdot))_\#\rho,(B_j(\rho',\cdot))_\#\rho'\bigr)
	\le
	C_{p,j}W_1(\rho,\rho').
	\]
\end{lemma}

\begin{proof}
	Let \(L_{g,j}\) be a Lipschitz constant of \(g_j\); by the ReLU feedforward
	definition, one may take \(L_{g,j}\le
	\|W_j^{(2)}\|_{\mathrm{op}}\|W_j^{(1)}\|_{\mathrm{op}}\). Define
	\(S_j(\rho,\mathbf z):=\mathbf z+\mathcal A_{j,\rho}(\mathbf z)\). By
	Lemma~\ref{lem:attention_stability},
	\(\|S_j(\rho,\mathbf z)-S_j(\rho',\mathbf z')\|_2
	\le (1+L_{A,j})\|\mathbf z-\mathbf z'\|_2+L_{A,j}W_1(\rho,\rho')\).
	Since \(B_j(\rho,\mathbf z)=S_j(\rho,\mathbf z)+g_j(S_j(\rho,\mathbf z))\), it follows that
	\[
	\|B_j(\rho,\mathbf z)-B_j(\rho',\mathbf z')\|_2
	\le
	(1+L_{g,j})\|S_j(\rho,\mathbf z)-S_j(\rho',\mathbf z')\|_2 .
	\]
	Thus the token-level estimate holds with
	\(C_{z,j}:=(1+L_{g,j})(1+L_{A,j})\) and
	\(C_{\rho,j}:=(1+L_{g,j})L_{A,j}\).
	
	It remains to prove the pushforward estimate. Let \(\pi\in\Gamma(\rho,\rho')\)
	be any coupling of \(\rho\) and \(\rho'\). Then
	\((B_j(\rho,\cdot),B_j(\rho',\cdot))_\#\pi\) is a coupling of
	\((B_j(\rho,\cdot))_\#\rho\) and \((B_j(\rho',\cdot))_\#\rho'\). Hence
	\[
	\begin{aligned}
		&W_1\bigl((B_j(\rho,\cdot))_\#\rho,(B_j(\rho',\cdot))_\#\rho'\bigr) \\
		&\qquad\le
		\int_{\mathcal Z\times\mathcal Z}
		\|B_j(\rho,\mathbf z)-B_j(\rho',\mathbf z')\|_2\,d\pi(\mathbf z,\mathbf z') \\
		&\qquad\le
		C_{z,j}\int_{\mathcal Z\times\mathcal Z}\|\mathbf z-\mathbf z'\|_2\,d\pi(\mathbf z,\mathbf z')
		+
		C_{\rho,j}W_1(\rho,\rho').
	\end{aligned}
	\]
	Taking the infimum over \(\pi\in\Gamma(\rho,\rho')\) gives the claim with
	\(C_{p,j}:=C_{z,j}+C_{\rho,j}\).
\end{proof}

\begin{proof}[Proof of Theorem~\ref{thm:transformer_predictor_joint_lip}]
	Let \(\mu_0:=\rho\), \(\mu_0':=\rho'\),
	\(\mathbf z_0:=\widetilde{\mathbf z}(\mathbf x)\), and
	\(\mathbf z_0':=\widetilde{\mathbf z}(\mathbf x')\). For \(j=1,\ldots,D\),
	define \(\mathbf z_j:=B_j(\mu_{j-1},\mathbf z_{j-1})\) and
	\(\mu_j:=(B_j(\mu_{j-1},\cdot))_\#\mu_{j-1}\), and define
	\(\mathbf z_j'\) and \(\mu_j'\) analogously from \(\mu_0'\) and \(\mathbf z_0'\).
	
	By Lemma~\ref{lem:block_stability},
	\(W_1(\mu_j,\mu_j')\le C_{p,j}W_1(\mu_{j-1},\mu_{j-1}')\). Iterating this
	inequality gives \(W_1(\mu_j,\mu_j')\le P_jW_1(\rho,\rho')\), where
	\(P_0:=1\) and \(P_j:=\prod_{r=1}^j C_{p,r}\). The token-level estimate in
	Lemma~\ref{lem:block_stability} further gives
	\[
	\|\mathbf z_j-\mathbf z_j'\|_2
	\le
	C_{z,j}\|\mathbf z_{j-1}-\mathbf z_{j-1}'\|_2
	+
	C_{\rho,j}P_{j-1}W_1(\rho,\rho'),
	\qquad j=1,\ldots,D .
	\]
	Iterating this recursion, there exists \(C_D>0\), depending only on the
	layerwise constants, such that
	\[
	\|\mathcal T_\rho^{(D)}(\mathbf z_0)-\mathcal T_{\rho'}^{(D)}(\mathbf z_0')\|_2
	\le
	C_D\bigl(\|\mathbf z_0-\mathbf z_0'\|_2+W_1(\rho,\rho')\bigr).
	\]
	Since the mask component is fixed, the product Euclidean metric gives
	\(\|\mathbf z_0-\mathbf z_0'\|_2=\|\mathbf x-\mathbf x'\|_2\). Therefore, by
	linearity of \(W_{\mathrm{out}}\),
	\[
	\begin{aligned}
		\|f(\rho,\mathbf x)-f(\rho',\mathbf x')\|_2
		&\le
		\|W_{\mathrm{out}}\|_{\mathrm{op}}
		\|\mathcal T_\rho^{(D)}(\mathbf z_0)-\mathcal T_{\rho'}^{(D)}(\mathbf z_0')\|_2 \\
		&\le
		\|W_{\mathrm{out}}\|_{\mathrm{op}}C_D
		\bigl(\|\mathbf x-\mathbf x'\|_2+W_1(\rho,\rho')\bigr).
	\end{aligned}
	\]
	The result follows with \(L_F:=\|W_{\mathrm{out}}\|_{\mathrm{op}}C_D\).
\end{proof}

	\subsection{Proof of Theorem~\ref{thm:icl_decomposition}}
	\label{sec:proof_icl_decomposition}
	
	\begin{proof}
		Throughout the proof, write
		\(C_{\ell,f}:=L_\ell\big(L_f+\sqrt{2}\max\{1,L_f\}\big)\).
		We first record an ERM reduction common to the two objectives. Let
		\(R_{T,N}^{\circ,\ell}\) denote either \(\widehat R_{T,N}^{\ell}\) or
		\(\bar R_{T,N}^{\ell}\), and let
		\(f^\circ\in\arg\min_{f\in\mathcal F}R_{T,N}^{\circ,\ell}(f)\). Since
		\(f^*\in\arg\min_{f\in\mathcal F}R_{\mathbb P}^{\ell}(f)\), we have
		\[
		\begin{aligned}
			R_{\mathbb P}^{\ell}(f^\circ)-R_{\mathbb P}^{\ell}(f^*)
			=&\,
			\bigl(R_{\mathbb P}^{\ell}(f^\circ)-R_{T,N}^{\circ,\ell}(f^\circ)\bigr)
			+
			\bigl(R_{T,N}^{\circ,\ell}(f^\circ)-R_{T,N}^{\circ,\ell}(f^*)\bigr) \\
			&+
			\bigl(R_{T,N}^{\circ,\ell}(f^*)-R_{\mathbb P}^{\ell}(f^*)\bigr) \\
			\le&\,
			\bigl(R_{\mathbb P}^{\ell}(f^\circ)-R_{T,N}^{\circ,\ell}(f^\circ)\bigr)
			+
			\bigl(R_{T,N}^{\circ,\ell}(f^*)-R_{\mathbb P}^{\ell}(f^*)\bigr) \\
			\le&\,
			2\sup_{f\in\mathcal F}
			\bigl|R_{\mathbb P}^{\ell}(f)-R_{T,N}^{\circ,\ell}(f)\bigr|,
		\end{aligned}
		\]
		where the first inequality follows from
		\(R_{T,N}^{\circ,\ell}(f^\circ)\le R_{T,N}^{\circ,\ell}(f^*)\). Hence
		\[
		\mathbb E[R_{\mathbb P}^{\ell}(f^\circ)]-R_{\mathbb P}^{\ell}(f^*)
		\le
		2\,\mathbb E\!\left[
		\sup_{f\in\mathcal F}
		\bigl|R_{\mathbb P}^{\ell}(f)-R_{T,N}^{\circ,\ell}(f)\bigr|
		\right].
		\]
		
		Define \(R_{\hat{\mathbb P}_T}^{\ell}(f):=
		T^{-1}\sum_{t=1}^T\int_{\mathcal Z}
		\ell\bigl(\mathbf y,f(\rho_t,\mathbf x)\bigr)\,
		d\rho_t(\mathbf x,\mathbf y)\) and
		\(\hat{\mathbb P}_{T,N}:=T^{-1}\sum_{t=1}^T\delta_{\hat\rho_t^{\,N}}\).
		For any \(f\in\mathcal F\),
		\[
		\bigl|R_{\mathbb P}^{\ell}(f)-R_{T,N}^{\circ,\ell}(f)\bigr|
		\le
		\bigl|R_{\mathbb P}^{\ell}(f)-R_{\hat{\mathbb P}_T}^{\ell}(f)\bigr|
		+
		\bigl|R_{\hat{\mathbb P}_T}^{\ell}(f)-R_{T,N}^{\circ,\ell}(f)\bigr|.
		\]
		By Theorem~\ref{thm:transferability},
		\[
		\sup_{f\in\mathcal F}
		\bigl|R_{\mathbb P}^{\ell}(f)-R_{\hat{\mathbb P}_T}^{\ell}(f)\bigr|
		\le
		C_{\ell,f}\,\mathbb W_1(\mathbb P,\hat{\mathbb P}_T).
		\]
		Moreover, \(R_{\hat{\mathbb P}_{T,N}}^{\ell}(f)
		=(TN)^{-1}\sum_{t=1}^T\sum_{j=1}^N
		\ell\bigl(\mathbf y_{t,j},f(\hat\rho_t^{\,N},\mathbf x_{t,j})\bigr)\).
		Another application of Theorem~\ref{thm:transferability} gives
		\[
		\sup_{f\in\mathcal F}
		\bigl|R_{\hat{\mathbb P}_T}^{\ell}(f)
		-
		R_{\hat{\mathbb P}_{T,N}}^{\ell}(f)\bigr|
		\le
		C_{\ell,f}\,\mathbb W_1(\hat{\mathbb P}_T,\hat{\mathbb P}_{T,N}).
		\]
		
		\emph{Masked ICL.}
		Now take \(R_{T,N}^{\circ,\ell}=\widehat R_{T,N}^{\ell}\) and
		\(f^\circ=\hat f\). 
		Since
		\[
		\widehat R_{T,N}^{\ell}(f)
		=
		\frac1{TN}\sum_{t=1}^T\sum_{j=1}^N
		\ell\bigl(\mathbf y_{t,j},
		f(\hat\rho_t^{\,(-j)},\mathbf x_{t,j})\bigr),
		\]
		the Lipschitz assumptions imply
		\[
		\begin{aligned}
			\bigl|R_{\hat{\mathbb P}_{T,N}}^{\ell}(f)
			-
			\widehat R_{T,N}^{\ell}(f)\bigr|
			&\le
			\frac1{TN}\sum_{t=1}^T\sum_{j=1}^N
			\Bigl|
			\ell\bigl(\mathbf y_{t,j},f(\hat\rho_t^{\,N},\mathbf x_{t,j})\bigr)
			-
			\ell\bigl(\mathbf y_{t,j},f(\hat\rho_t^{\,(-j)},\mathbf x_{t,j})\bigr)
			\Bigr| \\
			&\le
			\frac{L_\ell L_f}{TN}
			\sum_{t=1}^T\sum_{j=1}^N
			W_1(\hat\rho_t^{\,N},\hat\rho_t^{\,(-j)}).
		\end{aligned}
		\]
		Combining the preceding bounds, taking the supremum over \(f\in\mathcal F\),
		and then taking expectations yields
		\[
		\begin{aligned}
			\mathbb E[R_{\mathbb P}^{\ell}(\hat f)]
			-
			R_{\mathbb P}^{\ell}(f^*)
			\le\;&
			2C_{\ell,f}\,\mathbb E[\mathbb W_1(\mathbb P,\hat{\mathbb P}_T)]
			+
			2C_{\ell,f}\,\mathbb E[\mathbb W_1(\hat{\mathbb P}_T,\hat{\mathbb P}_{T,N})] \\
			&+
			\frac{2L_\ell L_f}{TN}
			\sum_{t=1}^T\sum_{j=1}^N
			\mathbb E\!\left[
			W_1(\hat\rho_t^{\,N},\hat\rho_t^{\,(-j)})
			\right].
		\end{aligned}
		\]
		
		\emph{Autoregressive ICL.}
		Next take \(R_{T,N}^{\circ,\ell}=\bar R_{T,N}^{\ell}\) and
		\(f^\circ=\bar f\). Since
		\(\bar R_{T,N}^{\ell}(f)
		=[T(N-1)]^{-1}\sum_{t=1}^T\sum_{j=2}^N
		\ell\bigl(\mathbf y_{t,j},f(\hat\rho_t^{(<j)},\mathbf x_{t,j})\bigr)\),
		adding and subtracting
		\([T(N-1)]^{-1}\sum_{t=1}^T\sum_{j=2}^N
		\ell\bigl(\mathbf y_{t,j},f(\hat\rho_t^{\,N},\mathbf x_{t,j})\bigr)\)
		gives
		\[
		\begin{aligned}
			\bigl|R_{\hat{\mathbb P}_{T,N}}^{\ell}(f)-\bar R_{T,N}^{\ell}(f)\bigr|
			\le\;&
			\frac{L_\ell L_f}{T(N-1)}
			\sum_{t=1}^T\sum_{j=2}^N
			W_1(\hat\rho_t^{\,N},\hat\rho_t^{(<j)}) \\
			&+
			\left|
			\frac1{TN}\sum_{t=1}^T\sum_{j=1}^N a_{t,j}
			-
			\frac1{T(N-1)}\sum_{t=1}^T\sum_{j=2}^N a_{t,j}
			\right|,
		\end{aligned}
		\]
		where \(a_{t,j}:=\ell\bigl(\mathbf y_{t,j},f(\hat\rho_t^{\,N},\mathbf x_{t,j})\bigr)\).
		Since \(0\le a_{t,j}\le M_\ell\),
		\[
		\begin{aligned}
			\left|
			\frac1{TN}\sum_{t=1}^T\sum_{j=1}^N a_{t,j}
			-
			\frac1{T(N-1)}\sum_{t=1}^T\sum_{j=2}^N a_{t,j}
			\right|
			&=
			\left|
			\frac1T\sum_{t=1}^T
			\left\{
			\frac{a_{t,1}}{N}
			-
			\frac1{N(N-1)}\sum_{j=2}^N a_{t,j}
			\right\}
			\right| \\
			&\le
			\frac{2M_\ell}{N}.
		\end{aligned}
		\]
		Thus
		\[
		\bigl|R_{\hat{\mathbb P}_{T,N}}^{\ell}(f)-\bar R_{T,N}^{\ell}(f)\bigr|
		\le
		\frac{L_\ell L_f}{T(N-1)}
		\sum_{t=1}^T\sum_{j=2}^N
		W_1(\hat\rho_t^{\,N},\hat\rho_t^{(<j)})
		+
		\frac{2M_\ell}{N}.
		\]
		Combining the preceding bounds, taking the supremum over \(f\in\mathcal F\),
		and then taking expectations yields
		\[
		\begin{aligned}
			\mathbb E[R_{\mathbb P}^{\ell}(\bar f)]
			-
			R_{\mathbb P}^{\ell}(f^*)
			\le\;&
			2C_{\ell,f}\,\mathbb E[\mathbb W_1(\mathbb P,\hat{\mathbb P}_T)]
			+
			2C_{\ell,f}\,\mathbb E[\mathbb W_1(\hat{\mathbb P}_T,\hat{\mathbb P}_{T,N})] \\
			&+
			\frac{2L_\ell L_f}{T(N-1)}
			\sum_{t=1}^T\sum_{j=2}^N
			\mathbb E\!\left[
			W_1(\hat\rho_t^{\,N},\hat\rho_t^{(<j)})
			\right]
			+
			\frac{4M_\ell}{N}.
		\end{aligned}
		\]
		The two bounds prove the theorem.
	\end{proof}
	
	\subsection{Proof of Corollary~\ref{cor:icl_rates}}
	\label{sec:proof_icl_rates}
	
	\begin{proof}
		Let \(D_{\mathcal Z}:=\operatorname{diam}(\mathcal Z)<\infty\). We bound
		the terms in Theorem~\ref{thm:icl_decomposition}.
		
		By Assumption~\ref{assum:wass_dim} and
		Lemma~\ref{lem:wass_convergence},
		\[
		\mathbb E\!\left[
		\mathbb W_1(\mathbb P,\hat{\mathbb P}_T)
		\right]
		\lesssim
		T^{-1/s}.
		\]
		
		Next, write \(\hat{\mathbb P}_T=T^{-1}\sum_{t=1}^T\delta_{\rho_t}\) and
		\(\hat{\mathbb P}_{T,N}=T^{-1}\sum_{t=1}^T
		\delta_{\hat\rho_t^{\,N}}\). The diagonal coupling between these two
		empirical meta-distributions gives
		\[
		\mathbb W_1(\hat{\mathbb P}_T,\hat{\mathbb P}_{T,N})
		\le
		\frac1T\sum_{t=1}^T W_1(\rho_t,\hat\rho_t^{\,N}).
		\]
		Taking expectations and applying the empirical \(W_1\) convergence rate on
		\(\mathcal Z=\mathcal X\times\mathcal Y\subseteq\mathbb R^{d_x+d_y}\)
		\citep{dudley1969speed} yields
		\[
		\mathbb E\!\left[
		\mathbb W_1(\hat{\mathbb P}_T,\hat{\mathbb P}_{T,N})
		\right]
		\lesssim
		N^{-1/(d_x+d_y)}.
		\]
		
		We now control the masked context term. Fix \(t\in[T]\). Since
		\(\hat\rho_t^{\,N}=N^{-1}\sum_{i=1}^N\delta_{\mathbf z_{t,i}}\) and
		\(\hat\rho_t^{\,(-1)}=(N-1)^{-1}\sum_{i=2}^N
		\delta_{\mathbf z_{t,i}}\), couple each common atom
		\(\mathbf z_{t,i}\), \(i\ge2\), to itself with mass \(1/N\), and send
		the remaining mass \(1/N\) at \(\mathbf z_{t,1}\) uniformly to
		\(\mathbf z_{t,2},\ldots,\mathbf z_{t,N}\). This coupling has the correct
		marginals, and hence
		\[
		W_1(\hat\rho_t^{\,N},\hat\rho_t^{\,(-1)})
		\le
		\frac1{N(N-1)}
		\sum_{i=2}^N \|\mathbf z_{t,1}-\mathbf z_{t,i}\|_2
		\le
		\frac{D_{\mathcal Z}}{N}.
		\]
		By relabeling, the same bound holds for every \(j\in[N]\). Therefore,
		\[
		\frac1{TN}\sum_{t=1}^T\sum_{j=1}^N
		\mathbb E\!\left[
		W_1(\hat\rho_t^{\,N},\hat\rho_t^{\,(-j)})
		\right]
		\lesssim
		N^{-1}.
		\]
		
		For the autoregressive context term, fix \(t\in[T]\) and
		\(j\in\{2,\ldots,N\}\). By the triangle inequality,
		\(W_1(\hat\rho_t^{\,N},\hat\rho_t^{(<j)})
		\le W_1(\hat\rho_t^{\,N},\rho_t)
		+ W_1(\rho_t,\hat\rho_t^{(<j)})\). The empirical \(W_1\) convergence rate
		gives \(\mathbb E[W_1(\hat\rho_t^{\,N},\rho_t)]
		\lesssim N^{-1/(d_x+d_y)}\). Moreover, since
		\(\hat\rho_t^{(<j)}=(j-1)^{-1}\sum_{i=1}^{j-1}
		\delta_{\mathbf z_{t,i}}\) is based on \(j-1\) i.i.d. samples from
		\(\rho_t\), the same bound gives
		\(\mathbb E[W_1(\rho_t,\hat\rho_t^{(<j)})]
		\lesssim (j-1)^{-1/(d_x+d_y)}\). Averaging over \(j\), we obtain
		\[
		\frac1{T(N-1)}
		\sum_{t=1}^T\sum_{j=2}^N
		\mathbb E\!\left[
		W_1(\hat\rho_t^{\,N},\hat\rho_t^{(<j)})
		\right]
		\lesssim
		N^{-1/(d_x+d_y)}
		+
		\frac1{N-1}\sum_{m=1}^{N-1}m^{-1/(d_x+d_y)}
		\lesssim
		N^{-1/(d_x+d_y)},
		\]
		where the last step uses \(d_x+d_y>1\). The finite-length term in
		Theorem~\ref{thm:icl_decomposition} contributes \(4M_\ell/N=O(N^{-1})\).
		
		Substituting these bounds into Theorem~\ref{thm:icl_decomposition} gives,
		for both \(f_{T,N}=\hat f\) and \(f_{T,N}=\bar f\),
		\[
		\mathbb E[R_{\mathbb P}^{\ell}(f_{T,N})]
		-
		R_{\mathbb P}^{\ell}(f^*)
		\lesssim
		T^{-1/s}
		+
		N^{-1/(d_x+d_y)}
		+
		N^{-1}.
		\]
		This proves the claim.
	\end{proof}

	\subsection{Proof of Theorem~\ref{thm:kshot_excess_risk}}
	\label{kshot}
	
	\begin{proof}
		Write \(f_{T,N}=\hat f\) for the masked objective and
		\(f_{T,N}=\bar f\) for the autoregressive objective. Let
		\(f^*\in\arg\min_{f\in\mathcal F}R_{\mathbb P}^{\ell}(f)\) and
		\(f_k^*\in\arg\min_{f\in\mathcal F}R_{k,\mathbb P}^{\ell}(f)\). For
		\(f\in\mathcal F\), define
		\(\Delta_k(f):=R_{k,\mathbb P}^{\ell}(f)-R_{\mathbb P}^{\ell}(f)\).
		Since \(f^*\) minimizes \(R_{\mathbb P}^{\ell}\) over \(\mathcal F\),
		\[
		\begin{aligned}
			R_{k,\mathbb P}^{\ell}(f_{T,N})
			-
			R_{k,\mathbb P}^{\ell}(f_k^*)
			&=
			\bigl[
			R_{\mathbb P}^{\ell}(f_{T,N})
			-
			R_{\mathbb P}^{\ell}(f^*)
			\bigr]
			+
			\bigl[
			R_{\mathbb P}^{\ell}(f^*)
			-
			R_{\mathbb P}^{\ell}(f_k^*)
			\bigr]
			+
			\Delta_k(f_{T,N})
			-
			\Delta_k(f_k^*)  \\
			&\le
			R_{\mathbb P}^{\ell}(f_{T,N})
			-
			R_{\mathbb P}^{\ell}(f^*)
			+
			2\sup_{f\in\mathcal F}|\Delta_k(f)|.
		\end{aligned}
		\]
		Taking expectations gives
		\[
		\mathbb E\!\left[
		R_{k,\mathbb P}^{\ell}(f_{T,N})
		-
		R_{k,\mathbb P}^{\ell}(f_k^*)
		\right]
		\le
		\mathbb E\!\left[
		R_{\mathbb P}^{\ell}(f_{T,N})
		-
		R_{\mathbb P}^{\ell}(f^*)
		\right]
		+
		2\sup_{f\in\mathcal F}|\Delta_k(f)|.
		\]
		
		It remains to bound the second term. For any \(f\in\mathcal F\), let
		\(\hat\rho_k:=k^{-1}\sum_{i=1}^k\delta_{\mathbf Z_i}\), where
		\(\mathbf Z_i=(\mathbf X_i,\mathbf Y_i)\) are i.i.d. from \(\rho\). By
		the definitions of \(R_{k,\mathbb P}^{\ell}\) and
		\(R_{\mathbb P}^{\ell}\), together with the Lipschitz assumptions on
		\(\ell\) and \(f\),
		\[
		\begin{aligned}
			|\Delta_k(f)|
			&\le
			\mathbb E_{\rho\sim\mathbb P}
			\mathbb E_{\mathbf Z_{1:k}\sim\rho^{\otimes k}}
			\mathbb E_{(\mathbf x,\mathbf y)\sim\rho}
			\Bigl|
			\ell\bigl(\mathbf y,f(\hat\rho_k,\mathbf x)\bigr)
			-
			\ell\bigl(\mathbf y,f(\rho,\mathbf x)\bigr)
			\Bigr|  \\
			&\le
			L_\ell L_f\,
			\mathbb E_{\rho\sim\mathbb P}
			\mathbb E_{\mathbf Z_{1:k}\sim\rho^{\otimes k}}
			W_1(\hat\rho_k,\rho)
			\lesssim
			k^{-1/(d_x+d_y)}.
		\end{aligned}
		\]
		The last inequality follows from the empirical \(W_1\) convergence rate on
		\(\mathcal Z=\mathcal X\times\mathcal Y\subseteq\mathbb R^{d_x+d_y}\)
		\citep{dudley1969speed}. Hence
		\[
		\sup_{f\in\mathcal F}
		\bigl|
		R_{k,\mathbb P}^{\ell}(f)
		-
		R_{\mathbb P}^{\ell}(f)
		\bigr|
		\lesssim
		k^{-1/(d_x+d_y)}.
		\]
		
		On the other hand, Corollary~\ref{cor:icl_rates} gives, for both the
		masked and autoregressive empirical minimizers,
		\[
		\mathbb E\!\left[
		R_{\mathbb P}^{\ell}(f_{T,N})
		-
		R_{\mathbb P}^{\ell}(f^*)
		\right]
		\lesssim
		T^{-1/s}
		+
		N^{-1/(d_x+d_y)}
		+
		N^{-1}.
		\]
		Combining the last two displays yields
		\[
		\mathbb E\!\left[
		R_{k,\mathbb P}^{\ell}(f_{T,N})
		-
		R_{k,\mathbb P}^{\ell}(f_k^*)
		\right]
		\lesssim
		T^{-1/s}
		+
		N^{-1/(d_x+d_y)}
		+
		k^{-1/(d_x+d_y)}
		+
		N^{-1}.
		\]
		
		Finally, if \(\mathbb P=\delta_\rho\), then
		\(\hat{\mathbb P}_T=\mathbb P\) almost surely, so the meta-level term
		\(\mathbb W_1(\mathbb P,\hat{\mathbb P}_T)\) in the preceding bound is
		zero. Hence the term \(T^{-1/s}\) vanishes.
	\end{proof}

    \subsection{Proof of Corollary~\ref{cor:budget_optimal_allocation}}
	\label{sec:proof_budget_optimal_allocation}
	\begin{proof}
		By Theorem~\ref{thm:kshot_excess_risk},
		\[
		\mathbb E\!\left[
		R_{k,\mathbb P}^{\ell}(f_{T,N})
		-
		R_{k,\mathbb P}^{\ell}(f_k^*)
		\right]
		\lesssim
		T^{-1/s}
		+
		N^{-1/(d_x+d_y)}
		+
		k^{-1/(d_x+d_y)}
		+
		N^{-1}.
		\]
		Since \(d_x+d_y\ge 1\), we have \(N^{-1}\le N^{-1/(d_x+d_y)}\)
		for \(N\ge 1\). Hence it suffices to minimize
		\(T^{-1/s}+N^{-1/(d_x+d_y)}\) subject to \(TN=B\). Writing
		\(N=B/T\), we obtain
		\[
		T^{-1/s}+N^{-1/(d_x+d_y)}
		=
		T^{-1/s}+B^{-1/(d_x+d_y)}T^{1/(d_x+d_y)} .
		\]
		Balancing the two terms yields
		\(T^{-1/s}\asymp B^{-1/(d_x+d_y)}T^{1/(d_x+d_y)}\), and therefore
		\(T\asymp B^{s/(s+d_x+d_y)}\). Consequently,
		\(N=B/T\asymp B^{(d_x+d_y)/(s+d_x+d_y)}\), and
		\(T^{-1/s}\asymp N^{-1/(d_x+d_y)}
		\asymp B^{-1/(s+d_x+d_y)}\). Substituting this into the bound gives
		\[
		\inf_{T,N:\,TN=B}
		\mathbb E\!\left[
		R_{k,\mathbb P}^{\ell}(f_{T,N})
		-
		R_{k,\mathbb P}^{\ell}(f_k^*)
		\right]
		\lesssim
		B^{-1/(s+d_x+d_y)}
		+
		k^{-1/(d_x+d_y)} .
		\]
		The integer constraint on \(T\) and \(N\) affects only the constants.
	\end{proof}
    
    \subsection{Proof of Theorem~\ref{thm:transferability}}
	\label{sec:proof_transferability}
	
	\begin{proof}
		For \(\rho\in\mathcal P(\mathcal Z)\), define
		\(F(\rho):=\int_{\mathcal Z}\ell(\mathbf y,f(\rho,\mathbf x))\,
		d\rho(\mathbf x,\mathbf y)\). Then
		\(R_{\mathbb P}^{\ell}(f)=\int F(\rho)\,d\mathbb P(\rho)\) and
		\(R_{\mathbb Q}^{\ell}(f)=\int F(\rho)\,d\mathbb Q(\rho)\).
		
		Fix \(\rho,\rho'\in\mathcal P(\mathcal Z)\) and
		\(\gamma\in\Gamma(\rho,\rho')\). Writing
		\(\mathbf z=(\mathbf x,\mathbf y)\) and
		\(\mathbf z'=(\mathbf x',\mathbf y')\), the coupling representation gives
		\[
		F(\rho)-F(\rho')
		=
		\int_{\mathcal Z\times\mathcal Z}
		\{\ell(\mathbf y,f(\rho,\mathbf x))
		-\ell(\mathbf y',f(\rho',\mathbf x'))\}
		\,d\gamma(\mathbf z,\mathbf z').
		\]
		By Assumptions~\ref{assum:loss_lipschitz} and
		\ref{assum:predictor_lipschitz},
		\[
		|F(\rho)-F(\rho')|
		\le
		L_\ell\int
		\{\|\mathbf y-\mathbf y'\|_2+L_f\|\mathbf x-\mathbf x'\|_2\}
		\,d\gamma
		+
		L_\ell L_f W_1(\rho,\rho').
		\]
		Since
		\(\|\mathbf y-\mathbf y'\|_2+L_f\|\mathbf x-\mathbf x'\|_2
		\le \sqrt2\max\{1,L_f\}\|\mathbf z-\mathbf z'\|_2\), we obtain
		\[
		|F(\rho)-F(\rho')|
		\le
		L_\ell\sqrt2\max\{1,L_f\}
		\int\|\mathbf z-\mathbf z'\|_2\,d\gamma
		+
		L_\ell L_f W_1(\rho,\rho').
		\]
		Taking the infimum over \(\gamma\in\Gamma(\rho,\rho')\) yields
		\[
		|F(\rho)-F(\rho')|
		\le
		C_{\ell,f}W_1(\rho,\rho'),
		\qquad
		C_{\ell,f}:=L_\ell\bigl(L_f+\sqrt2\max\{1,L_f\}\bigr).
		\]
		
		Now let \(\pi\in\Gamma(\mathbb P,\mathbb Q)\). Then
		\[
		\begin{aligned}
			|R_{\mathbb P}^{\ell}(f)-R_{\mathbb Q}^{\ell}(f)|
			&=
			\left|\int_{\mathcal P(\mathcal Z)\times\mathcal P(\mathcal Z)}
			\{F(\rho)-F(\rho')\}\,d\pi(\rho,\rho')\right| \\
			&\le
			C_{\ell,f}
			\int_{\mathcal P(\mathcal Z)\times\mathcal P(\mathcal Z)}
			W_1(\rho,\rho')\,d\pi(\rho,\rho').
		\end{aligned}
		\]
		Taking the infimum over \(\pi\in\Gamma(\mathbb P,\mathbb Q)\) and using
		Definition~\ref{def:lifted_W1} gives
		\[
		|R_{\mathbb P}^{\ell}(f)-R_{\mathbb Q}^{\ell}(f)|
		\le
		C_{\ell,f}\,\mathbb W_1(\mathbb P,\mathbb Q).
		\]\end{proof}

	\subsection{Proof of Theorem~\ref{thm:kshot_transferability_shift}}
	\label{sec:proof_kshot_transferability_shift}
	
	\begin{proof}
		Write \(C_{\ell,f}:=L_\ell\big(L_f+\sqrt2\max\{1,L_f\}\big)\).
		We first prove the \(k\)-shot analogue of
		Theorem~\ref{thm:transferability}. Fix any \(f\in\mathcal F\), and define
		\(F_k(\rho):=
		\mathbb E_{(\mathbf z,\mathbf z_1,\ldots,\mathbf z_k)\sim\rho^{\otimes(k+1)}}
		[\ell(\mathbf y,f(\hat\rho_k,\mathbf x))]\), where
		\(\hat\rho_k:=k^{-1}\sum_{i=1}^k\delta_{\mathbf z_i}\). Then
		\(R_{k,\mathbb P}^{\ell}(f)=\int F_k(\rho)\,d\mathbb P(\rho)\) and
		\(R_{k,\mathbb Q}^{\ell}(f)=\int F_k(\rho)\,d\mathbb Q(\rho)\).
		
		Fix \(\rho,\rho'\in\mathcal P(\mathcal Z)\) and
		\(\gamma\in\Gamma(\rho,\rho')\). Let
		\(((\mathbf z,\mathbf z'),(\mathbf z_1,\mathbf z_1'),\ldots,
		(\mathbf z_k,\mathbf z_k'))\sim\gamma^{\otimes(k+1)}\), and set
		\(\hat\rho_k':=k^{-1}\sum_{i=1}^k\delta_{\mathbf z_i'}\). By the Lipschitz
		conditions on \(\ell\) and \(f\),
		\[
		|F_k(\rho)-F_k(\rho')|
		\le
		L_\ell\,
		\mathbb E\!\left[
		\|\mathbf y-\mathbf y'\|_2
		+
		L_f\|\mathbf x-\mathbf x'\|_2
		+
		L_f W_1(\hat\rho_k,\hat\rho_k')
		\right].
		\]
		Moreover,
		\(W_1(\hat\rho_k,\hat\rho_k')\le k^{-1}\sum_{i=1}^k
		\|\mathbf z_i-\mathbf z_i'\|_2\), and hence
		\(\mathbb E W_1(\hat\rho_k,\hat\rho_k')
		\le \int\|\mathbf z-\mathbf z'\|_2\,d\gamma(\mathbf z,\mathbf z')\).
		Also,
		\(\|\mathbf y-\mathbf y'\|_2+L_f\|\mathbf x-\mathbf x'\|_2
		\le \sqrt2\max\{1,L_f\}\|\mathbf z-\mathbf z'\|_2\). Therefore,
		\[
		|F_k(\rho)-F_k(\rho')|
		\le
		C_{\ell,f}
		\int\|\mathbf z-\mathbf z'\|_2\,d\gamma(\mathbf z,\mathbf z').
		\]
		Taking the infimum over \(\gamma\in\Gamma(\rho,\rho')\) gives
		\(|F_k(\rho)-F_k(\rho')|\le C_{\ell,f}W_1(\rho,\rho')\). Integrating this
		bound over any \(\pi\in\Gamma(\mathbb P,\mathbb Q)\) and then taking the
		infimum over \(\pi\) yields
		\[
		\big|R_{k,\mathbb P}^{\ell}(f)-R_{k,\mathbb Q}^{\ell}(f)\big|
		\le
		C_{\ell,f}\,\mathbb W_1(\mathbb P,\mathbb Q).
		\]
		
		Now let
		\(f_{k,\mathbb P}^*\in\arg\min_{f\in\mathcal F}R_{k,\mathbb P}^{\ell}(f)\).
		By adding and subtracting the corresponding \(\mathbb P\)-risks, and using
		the optimality of \(f_{k,\mathbb P}^*\) under \(R_{k,\mathbb P}^{\ell}\),
		\[
		\begin{aligned}
			R_{k,\mathbb Q}^{\ell}(f_{T,N})
			-
			R_{k,\mathbb Q}^{\ell}(f_{k,\mathbb Q}^*)
			&\le
			R_{k,\mathbb P}^{\ell}(f_{T,N})
			-
			R_{k,\mathbb P}^{\ell}(f_{k,\mathbb P}^*)  \\
			&\quad+
			2C_{\ell,f}\,\mathbb W_1(\mathbb P,\mathbb Q).
		\end{aligned}
		\]
		Taking expectation and applying Theorem~\ref{thm:kshot_excess_risk} under
		the pretraining distribution \(\mathbb P\) gives
		\[
		\mathbb E\!\left[
		R_{k,\mathbb Q}^{\ell}(f_{T,N})
		-
		R_{k,\mathbb Q}^{\ell}(f_{k,\mathbb Q}^*)
		\right]
		\lesssim
		T^{-1/s}
		+
		N^{-1/(d_x+d_y)}
		+
		k^{-1/(d_x+d_y)}
		+
		N^{-1}
		+
		2C_{\ell,f}\,\mathbb W_1(\mathbb P,\mathbb Q).
		\]
		This proves the theorem.
	\end{proof}

	\section{Helper Wasserstein Concentration Inequalities}
  
    The following empirical Wasserstein bound is a restatement of
    \citep[Theorem~10]{chakraborty2026generalization}.
    \begin{lemma}\label{lem:wass_convergence}
    Let $\rho$ be a Borel probability measure on $\mathcal Z$, and let
    $\hat\rho^{\,N}$ denote the corresponding empirical measure. Assume that
    $\rho$ has finite $q$-th moment and let $0<p<q$. Then, for any
    $d>d_{p,q}^*(\rho)$, there exist constants $N_0\in\mathbb N$ and $c>0$,
    possibly depending on $d,\rho,p$ and $q$, such that, for all $N\ge N_0$,
    \[
    \mathbb E\,W_p^p(\rho,\hat\rho^{\,N})
    \le
    cN^{-p/d}.
    \]
    \end{lemma}

	\begin{lemma}[Concentration of Wasserstein metric on a manifold]
		\label{lem:wasserstein_concentration}
		Let \(\mathcal M\subseteq\mathcal Z\) be a compact \(C^1\) Riemannian
		submanifold of intrinsic dimension \(d_{\mathrm{int}}>2\), equipped with
		the metric inherited from the Euclidean norm on \(\mathcal Z\). Let \(\rho\)
		be a Borel probability measure on \(\mathcal M\), and let
		\(\widehat\rho^{\,N}\) denote the empirical measure based on \(N\) i.i.d.\
		samples from \(\rho\). Then, for every \(\epsilon>0\) and \(N\in\mathbb N\),
		\begin{equation}
			\label{eq:w1_mcdiarmid}
			\mathbb P\!\left(
			\left|
			W_1(\widehat\rho^{\,N},\rho)
			-
			\mathbb E\!\left[W_1(\widehat\rho^{\,N},\rho)\right]
			\right|
			\ge \epsilon
			\right)
			\le
			2\exp\!\left(
			-\frac{2N\epsilon^2}{\operatorname{diam}(\mathcal M)^2}
			\right).
		\end{equation}
		Moreover, there exists a constant \(C_{\mathcal M}>0\) such that
		\begin{equation}
			\label{eq:wasserstein_expectation}
			\mathbb E\!\left[
			W_1(\widehat\rho^{\,N},\rho)
			\right]
			\le
			C_{\mathcal M}\operatorname{diam}(\mathcal M)N^{-1/d_{\mathrm{int}}}.
		\end{equation}
	\end{lemma}

	\begin{proof}
		Let \(D_{\mathcal M}:=\operatorname{diam}(\mathcal M)\).

		\paragraph{Step 1: \(\dim_A(\mathcal M)=d_{\mathrm{int}}\).}
		For each \(z\in\mathcal M\), choose a chart \((U_z,\phi_z)\) such that
		\(\phi_z:U_z\to \phi_z(U_z)\subset\mathbb R^{d_{\mathrm{int}}}\) is a
		\(C^1\) diffeomorphism. Choose an open set \(V_z\) with \(z\in V_z\) and
		\(\overline V_z\subset U_z\). Since \(\overline V_z\) is compact and
		\(\phi_z,\phi_z^{-1}\) are \(C^1\), both maps are Lipschitz on the relevant
		compact sets. Hence \(\phi_z\) is bi-Lipschitz on \(\overline V_z\). By
		Lemma~\ref{lem:assouad_properties}(v),
		\[
		\dim_A(\overline V_z)
		=
		\dim_A\!\big(\phi_z(\overline V_z)\big).
		\]
		Since \(\phi_z(V_z)\) is open in \(\mathbb R^{d_{\mathrm{int}}}\), there is
		\(r_z>0\) such that \(B(\phi_z(z),r_z)\subset \phi_z(V_z)\). By
		Lemma~\ref{lem:assouad_properties}(i) and (iii),
		\[
		d_{\mathrm{int}}
		=
		\dim_A\!\big(B(\phi_z(z),r_z)\big)
		\le
		\dim_A\!\big(\phi_z(\overline V_z)\big)
		\le
		\dim_A\!\big(\phi_z(U_z)\big)
		=
		d_{\mathrm{int}}.
		\]
		Therefore \(\dim_A(\overline V_z)=d_{\mathrm{int}}\). By compactness,
		finitely many sets \(V_{z_1},\ldots,V_{z_K}\) cover \(\mathcal M\). Since
		\(\mathcal M=\bigcup_{k=1}^K\overline V_{z_k}\), Lemma~\ref{lem:assouad_properties}(ii)
		gives \(\dim_A(\mathcal M)=d_{\mathrm{int}}\).
		
		\paragraph{Step 2: Expected Wasserstein rate.}
		The finite bi-Lipschitz atlas in Step~1 implies that there exists a constant
		\(K_{\mathcal M}\ge 1\), depending only on \(\mathcal M\), such that
		\(N_{\mathcal M}^{\mathrm{cov}}(r)\le
		K_{\mathcal M}(D_{\mathcal M}/r)^{d_{\mathrm{int}}}\) for all
		\(0<r\le D_{\mathcal M}\).
		Here the compactness of \(\mathcal M\) is essential: the asymptotic control on
		covering numbers provided by the Assouad dimension is extended to all scales
		by taking the maximum of the finite per-chart bi-Lipschitz constants, yielding
		a uniform bound on the whole manifold.
		Since \(d_{\mathrm{int}}>2\), applying
		Lemma~\ref{lem:mean_Wp} with \(E=\mathcal M\), \(p=1\),
		\(\alpha=d_{\mathrm{int}}\), \(D_E=D_{\mathcal M}\), and
		\(K_E=K_{\mathcal M}\) yields
		\[
		\mathbb E\!\left[
		W_1(\widehat\rho^{\,N},\rho)
		\right]
		\le
		c\left(\frac{2}{d_{\mathrm{int}}-2}\right)^{2/d_{\mathrm{int}}}
		D_{\mathcal M}K_{\mathcal M}^{1/d_{\mathrm{int}}}
		N^{-1/d_{\mathrm{int}}}.
		\]
		Absorbing
		\(c(2/(d_{\mathrm{int}}-2))^{2/d_{\mathrm{int}}}
		K_{\mathcal M}^{1/d_{\mathrm{int}}}\) into \(C_{\mathcal M}\) proves
		\eqref{eq:wasserstein_expectation}.

		\paragraph{Step 3: Concentration.}
		Since \(D_{\mathcal M}<\infty\) and \(\rho\) is a Borel probability measure
		on the compact, hence Polish, space \(\mathcal M\), Proposition~20 of
		\citet{weed2019sharp} applies. Consequently, for every \(\epsilon>0\),
		\[
		\mathbb P\!\left(
		\left|
		W_1(\widehat\rho^{\,N},\rho)
		-
		\mathbb E\!\left[W_1(\widehat\rho^{\,N},\rho)\right]
		\right|
		\ge \epsilon
		\right)
		\le
		2\exp\!\left(
		-\frac{2N\epsilon^2}{D_{\mathcal M}^2}
		\right),
		\]
		which proves \eqref{eq:w1_mcdiarmid}.
	\end{proof}

	\begin{lemma}[Properties of the Assouad dimension {\citep[Lem.~9.6]{robinson2010dimensions}}]
		\label{lem:assouad_properties}
		Let \((E,d)\) be a metric space and let \(A,B\subseteq E\). Denote by
		\(\dim_A(\cdot)\) the Assouad dimension. Then:
		\begin{enumerate}
			\item[(i)] If \(A\subseteq B\), then \(\dim_A(A)\le \dim_A(B)\).
			\item[(ii)] \(\dim_A(A\cup B)=\max\{\dim_A(A),\dim_A(B)\}\).
			\item[(iii)] If \(U\) is a nonempty open subset of \(\mathbb R^q\), then
			\(\dim_A(U)=q\).
			\item[(iv)] If \(E\) is compact, then \(\dim_B(E)\le\dim_A(E)\), where
			\(\dim_B(\cdot)\) denotes the upper box-counting dimension.
			\item[(v)] The Assouad dimension is invariant under bi-Lipschitz mappings.
		\end{enumerate}
	\end{lemma}

	\begin{lemma}[Mean Wasserstein-\(p\) rate under polynomial covering]
		\label{lem:mean_Wp}
		Let \((E,d)\) be a Polish metric space with
		\(D_E:=\operatorname{diam}(E)<\infty\), and let \(\rho\) be a Borel
		probability measure on \(E\). Fix \(p\ge 1\). Assume that there exist
		constants \(K_E>0\) and \(\alpha>2p\) such that
		\(N_E^{\mathrm{cov}}(r)\le K_E(D_E/r)^\alpha\) for all
		\(0<r\le D_E\). Let \(\widehat\rho^{\,N}\) denote the empirical measure
		based on \(N\) i.i.d.\ samples drawn from \(\rho\). Then there exists a
		universal constant \(c\le 64/3\) such that
		\[
		\mathbb E\!\left[
		W_p(\widehat\rho^{\,N},\rho)
		\right]
		\le
		c\left(\frac{2p}{\alpha-2p}\right)^{2p/\alpha}
		D_EK_E^{1/\alpha}N^{-1/\alpha}.
		\]
		This follows directly from Corollary~1.2 of
		\citet{boissard2014mean}.
	\end{lemma}

	\begin{lemma}[Wasserstein bound under \((m,\Delta)\)-clusterability]
		\label{thm:clusterable_wasserstein_D}
		Let \((\mathcal Z,\mathrm d_{\mathcal Z})\) be a bounded Polish metric space with
		\(\operatorname{diam}(\mathcal Z)\le D_{\mathcal Z}\), where
		\(D_{\mathcal Z}>0\), and let \(p\ge 1\).
		Let \(\rho\) be a Borel probability measure on \(\mathcal Z\), and let
		\(\widehat\rho^{\,n}:=n^{-1}\sum_{i=1}^n\delta_{Z_i}\), where
		\(Z_1,\ldots,Z_n\) are i.i.d.\ samples from \(\rho\). Assume that \(\rho\) is
		\((m,\Delta)\)-clusterable, with \(\Delta>0\), in the sense that there exist
		points \(z_1,\ldots,z_m\in\mathcal Z\) such that
		\(\operatorname{supp}(\rho)\subseteq \bigcup_{j=1}^m B_{\mathcal Z}(z_j,\Delta)\).
		If \(n\le m(D_{\mathcal Z}/(2\Delta))^{2p}\), then
		\[
		\mathbb E\!\left[
		W_p^p(\rho,\widehat\rho^{\,n})
		\right]
		\le
		\left(2^{p-1}+2^{p-2}\right)D_{\mathcal Z}^p\sqrt{\frac{m}{n}}.
		\]
	\end{lemma}
	
	\begin{proof}
		The proof follows the clusterability idea of
		\citet[Proposition~13]{weed2019sharp} through a direct projection onto cluster
		centers, while keeping the diameter dependence explicit. Repeated centers may be
		merged without increasing \(m\), and zero-mass cells can be kept since they do
		not contribute to the estimates.
		
		After merging repeated centers if necessary, define a measurable partition of
		\(\operatorname{supp}(\rho)\) by
		\(A_1:=\operatorname{supp}(\rho)\cap B_{\mathcal Z}(z_1,\Delta)\) and, for \(2\le j\le m\),
		\[
		A_j:=
		\operatorname{supp}(\rho)\cap B_{\mathcal Z}(z_j,\Delta)
		\setminus \bigcup_{\ell<j}A_\ell .
		\]
		Then \(A_j\subseteq B_{\mathcal Z}(z_j,\Delta)\) and
		\(\operatorname{supp}(\rho)=\bigcup_{j=1}^m A_j\). Define \(\pi(z)=z_j\) for
		\(z\in A_j\), and extend \(\pi\) arbitrarily to
		\(\mathcal Z\setminus\operatorname{supp}(\rho)\). Since
		\(\rho(\operatorname{supp}(\rho))=1\), this extension does not affect
		\(\pi_\#\rho\). Set \(\nu:=\pi_\#\rho\), \(Y_i:=\pi(Z_i)\), and
		\(\widehat\nu_n:=n^{-1}\sum_{i=1}^n\delta_{Y_i}\). Then
		\(Y_1,\ldots,Y_n\) are i.i.d.\ with law \(\nu\).
		
		Since \(\mathrm d_{\mathcal Z}(z,\pi(z))\le\Delta\) for
		\(z\in\operatorname{supp}(\rho)\), the map
		\(z\mapsto\pi(z)\) induces a coupling between \(\rho\) and \(\nu\), and hence
		\(W_p^p(\rho,\nu)\le\int \mathrm d_{\mathcal Z}^p(z,\pi(z))\,d\rho(z)\le\Delta^p\).
		Similarly, the empirical coupling \(n^{-1}\sum_{i=1}^n\delta_{(Z_i,Y_i)}\) gives
		\(W_p^p(\widehat\rho^{\,n},\widehat\nu_n)\le
		n^{-1}\sum_{i=1}^n \mathrm d_{\mathcal Z}^p(Z_i,Y_i)\le\Delta^p\).
		
		Write \(\nu=\sum_{j=1}^m\alpha_j\delta_{z_j}\) and
		\(\widehat\nu_n=\sum_{j=1}^m\widehat\alpha_j\delta_{z_j}\), where
		\(\alpha_j=\rho(A_j)\) and
		\(\widehat\alpha_j=n^{-1}\sum_{i=1}^n\mathbf 1_{\{Y_i=z_j\}}\). Since the
		centers have diameter at most \(D_{\mathcal Z}\), keeping the common mass at each center
		fixed and transporting only the unmatched mass gives
		\[
		W_p^p(\nu,\widehat\nu_n)
		\le
		D_{\mathcal Z}^p\|\nu-\widehat\nu_n\|_{\mathrm{TV}}
		=
		\frac{D_{\mathcal Z}^p}{2}\sum_{j=1}^m|\alpha_j-\widehat\alpha_j|.
		\]
		For each \(j\), \(n\widehat\alpha_j\sim\mathrm{Bin}(n,\alpha_j)\) marginally,
		and therefore
		\[
		\mathbb E|\widehat\alpha_j-\alpha_j|
		\le
		\sqrt{\operatorname{Var}(\widehat\alpha_j)}
		=
		\sqrt{\frac{\alpha_j(1-\alpha_j)}{n}}
		\le
		\sqrt{\frac{\alpha_j}{n}}.
		\]
		Taking expectations and using Cauchy--Schwarz, we obtain
		\[
		\mathbb E\!\left[
		W_p^p(\nu,\widehat\nu_n)
		\right]
		\le
		\frac{D_{\mathcal Z}^p}{2\sqrt n}\sum_{j=1}^m\sqrt{\alpha_j}
		\le
		\frac{D_{\mathcal Z}^p}{2}\sqrt{\frac{m}{n}},
		\]
		where the last step uses \(\sum_{j=1}^m\alpha_j=1\).
		
		By the triangle inequality for \(W_p\),
		\[
		W_p(\rho,\widehat\rho^{\,n})
		\le
		W_p(\rho,\nu)+W_p(\nu,\widehat\nu_n)
		+W_p(\widehat\nu_n,\widehat\rho^{\,n})
		\le
		2\Delta+W_p(\nu,\widehat\nu_n).
		\]
		Using \((a+b)^p\le 2^{p-1}(a^p+b^p)\), with \(a=2\Delta\) and
		\(b=W_p(\nu,\widehat\nu_n)\), gives
		\[
		W_p^p(\rho,\widehat\rho^{\,n})
		\le
		2^{p-1}(2\Delta)^p
		+
		2^{p-1}W_p^p(\nu,\widehat\nu_n).
		\]
		After taking expectations,
		\[
		\mathbb E\!\left[
		W_p^p(\rho,\widehat\rho^{\,n})
		\right]
		\le
		2^{p-1}(2\Delta)^p
		+
		2^{p-2}D_{\mathcal Z}^p\sqrt{\frac{m}{n}}.
		\]
		Finally, \(n\le m(D_{\mathcal Z}/(2\Delta))^{2p}\) is equivalent to
		\((2\Delta)^p\le D_{\mathcal Z}^p\sqrt{m/n}\). Substituting this bound into the previous
		display gives
		\[
		\mathbb E\!\left[
		W_p^p(\rho,\widehat\rho^{\,n})
		\right]
		\le
		\left(2^{p-1}+2^{p-2}\right)D_{\mathcal Z}^p\sqrt{\frac{m}{n}},
		\]
		as claimed.
	\end{proof}

    \section{Synthetic In-Context Function Learning Experiments}
	\label{app:synthetic-icl}
	
	This appendix gives implementation details for the synthetic experiments. At
	test time, model parameters are fixed, and prediction is made only from the
	input--output examples in the prompt and the query input.
	
	\paragraph{Task distributions.}
	We use four task families. For noiseless linear regression, each prompt samples
	\(w\sim\mathcal N(0,I_d)\), \(x_i\sim\mathcal N(0,I_d)\), and
	\(y_i=w^\top x_i\), with \(d=20\). All examples in the same prompt share the
	same \(w\), while different prompts use independent latent vectors. For noisy
	linear regression, labels are generated as
	\(y_i^{\mathrm{raw}}=w^\top x_i+\epsilon_i\), where
	\(\epsilon_i\sim\mathcal N(0,\sigma^2)\) and \(\sigma=1\). The main noisy-linear
	curve uses population normalization
	\(y_i=y_i^{\mathrm{raw}}\sqrt d/\sqrt{d+\sigma^2}\), evaluates the noiseless
	linear-regression checkpoints without additional noisy training, and measures
	error against the normalized noisy query label. We therefore use this setting as
	a robustness evaluation of the noiselessly trained checkpoints. For
	decision-tree regression, each prompt samples an independent depth-\(4\)
	regression tree whose internal nodes test \(x_j>0\) for randomly selected
	coordinates and whose leaf values are standard normal. For two-layer ReLU
	regression, each prompt samples
	\(f(x)=\sqrt{2/h}\sum_{j=1}^h a_j\operatorname{ReLU}(u_j^\top x)\), with
	\(d=20\), \(h=100\), \(u_j\sim\mathcal N(0,I_d)\), and
	\(a_j\sim\mathcal N(0,1)\).
	
	\paragraph{Prompt format.}
	For context size \(n\), the prompt is
	\((x_1,y_1),\ldots,(x_n,y_n),x_{n+1}\), and the model predicts \(y_{n+1}\). The
	case \(n=0\) is a no-context query. Evaluation prompts are sampled independently
	from training prompts, and no parameter updates are performed at test time.
	
	\paragraph{Masked Pair Encoder.}
	The masked model represents each \((x_i,y_i)\) as one pair token. For an
	observed label, the token embedding is \(h_i=W_xx_i+W_yy_i\); for a masked label,
	it is \(h_i=W_xx_i+e_{\mathrm{mask}}\), where \(e_{\mathrm{mask}}\) is learned.
	The pair sequence is processed by a bidirectional Transformer encoder, and a
	scalar readout predicts the labels at masked positions. In the set-encoder
	variant used in the main runs, all position identifiers are set to zero, reducing
	absolute-position shortcuts.
	
	\paragraph{Masked label-prediction objective.}
	Training uses online-generated synthetic prompts. For each sampled sequence,
	\(K=8\) target positions are sampled, and the implementation expands the sequence
	into \(K\) single-mask copies. Each expanded example masks exactly one label,
	while all non-target labels remain visible. The loss is the squared error on the
	masked target, \(\mathcal L=\mathbb E[(\hat y_t-y_t)^2]\). Thus \(K=8\)
	increases the number of supervised targets obtained from each sampled prompt,
	while each forward pass remains a single-mask prediction problem. In the variable
	dense leave-one-out runs, the training prefix length is sampled during training
	under the current curriculum limit. At evaluation time, all support labels are
	observed and only \(y_{n+1}\) is replaced by \(e_{\mathrm{mask}}\); the true query
	label is used only for the metric.
	
	\paragraph{GPT-2-style causal Transformer baseline.}
	The causal baseline is a GPT-2-style Transformer trained from scratch on the
	same synthetic task family, following the synthetic in-context function-learning
	protocol of \citet{garg2022can}. It interleaves input and label tokens and
	predicts each label from the hidden state of the corresponding input token.
	
	\paragraph{Baselines.}
	For linear regression, we report OLS and 3NN. OLS is fit separately for each
	prompt as \(\hat w_n=X_n^\dagger y_n\), and the prediction is
	\(\hat y_{n+1}=x_{n+1}^\top\hat w_n\), with no intercept or ridge
	regularization. When \(n<d\), this is the underdetermined least-squares solution;
	when \(n=0\), the prediction is zero. The 3NN baseline averages the labels of the
	\(\min\{3,n\}\) nearest context inputs in Euclidean distance and predicts zero at
	\(n=0\). For noisy linear regression, we additionally include the averaging
	estimator \(\hat w_n=n^{-1}\sum_{i=1}^n x_i y_i\). For decision trees, we include
	nearest neighbors, greedy regression trees, sign-preprocessed greedy trees, and
	sign-preprocessed XGBoost. The greedy tree baseline fits a
	\texttt{DecisionTreeRegressor} separately for each prompt with maximum depth
	\(4\). The sign-preprocessed baselines replace each input by
	\(\operatorname{sign}(x)\). The XGBoost baseline is also fit separately for each
	prompt on \(\operatorname{sign}(x)\), using squared-error regression, \(100\)
	estimators, maximum depth \(4\), learning rate \(0.1\), unit subsampling and
	column sampling, and \(\ell_2\) regularization parameter \(1\). For two-layer
	ReLU networks, we include nearest neighbors and a per-prompt two-layer ReLU
	network fit on the in-context examples. Following prior synthetic ICL naming, the
	figure labels this reference as ``2-layer NN, GD''; concretely, the released
	implementation uses a two-layer ReLU network with \(100\) hidden units optimized
	for \(100\) steps with Adam at learning rate \(5\times 10^{-3}\).
	
	\paragraph{Evaluation metric.}
	For each \(n\), we estimate query risk by averaging over held-out prompts,
	\(\widehat R_n=M^{-1}\sum_{m=1}^M(\hat y_{n+1}^{(m)}-y_{n+1}^{(m)})^2\). For
	noiseless linear regression, noisy linear regression, and two-layer ReLU
	regression, the released figures report \(\widehat R_n/d\). For decision-tree
	regression, they report raw mean squared error because the leaf variance is
	already order one.
	
	\begin{table}[t]
		\centering
		\caption{Main synthetic ICL training settings. All neural models are trained
			from scratch on online-generated prompts with batch size \(64\), learning
			rate \(10^{-4}\), hidden width \(256\), and \(8\) attention heads. ``Max
			train points'' denotes the maximum number of sampled points in the training
			curriculum, including the query point; ``Eval. \(n\)'' denotes the plotted
			number of in-context examples. The noisy-linear experiment evaluates the
			linear-regression checkpoints without additional noisy training.}
		\label{tab:synthetic-icl-settings}
		\small
		\begin{tabular}{llccccl}
			\toprule
			Task & Model & Layers & Max train points & Eval. \(n\) & Steps & Objective \\
			\midrule
			Linear & Masked Pair Encoder & 6 & 41 & \(0\)--\(40\) & 200k & masked, \(K=8\) \\
			Linear & Causal Transformer & 12 & 41 & \(0\)--\(40\) & 500k & causal \\
			Decision tree & Masked Pair Encoder & 12 & 101 & \(1\)--\(40\) & 300k & masked, \(K=8\) \\
			Decision tree & Causal Transformer & 12 & 101 & \(1\)--\(40\) & 200k & causal \\
			Two-layer ReLU & Masked Pair Encoder & 12 & 101 & \(0\)--\(100\) & 300k & masked, \(K=8\) \\
			Two-layer ReLU & Causal Transformer & 12 & 101 & \(0\)--\(100\) & 300k & causal \\
			\bottomrule
		\end{tabular}
	\end{table}
	
	\paragraph{No label leakage.}
	For the Masked Pair Encoder, the target label is replaced by
	\(e_{\mathrm{mask}}\) before the Transformer forward pass, so bidirectional
	attention cannot reveal the true target label. For the causal Transformer,
	predictions are read from input-token hidden states, and causal masking prevents
	access to the label token being predicted. In both cases, the true query label is
	used only for the loss or evaluation metric.
    \paragraph{Reproducibility.}
	The code and released artifacts are available at
	\url{https://github.com/ari-arden/masked-function-icl}. The repository includes the final
	experimental results, curve CSV files, plotting scripts, and scripts for
	regenerating the figures reported in this paper.

	\vskip 0.2in
	\bibliography{ICL}

@book{vapnik2013nature,
  title={The nature of statistical learning theory},
  author={Vapnik, Vladimir},
  year={2013},
  publisher={Springer science \& business media}
}

@article{maurer2016benefit,
	title={The benefit of multitask representation learning},
	author={Maurer, Andreas and Pontil, Massimiliano and Romera-Paredes, Bernardino},
	journal={Journal of Machine Learning Research},
	volume={17},
	number={81},
	pages={1--32},
	year={2016}
}

@article{szabo2016learning,
	title={Learning theory for distribution regression},
	author={Szab{\'o}, Zolt{\'a}n and Sriperumbudur, Bharath K and P{\'o}czos, Barnab{\'a}s and Gretton, Arthur},
	journal={Journal of Machine Learning Research},
	volume={17},
	number={152},
	pages={1--40},
	year={2016}
}

@inproceedings{poczos2013distribution,
	title={Distribution-free distribution regression},
	author={P{\'o}czos, Barnab{\'a}s and Singh, Aarti and Rinaldo, Alessandro and Wasserman, Larry},
	booktitle={artificial intelligence and statistics},
	pages={507--515},
	year={2013},
	organization={PMLR}
}

@article{weed2019sharp,
	title={Sharp asymptotic and finite-sample rates of convergence of empirical measures in Wasserstein distance},
	author={Weed, Jonathan and Bach, Francis},
	journal={Bernoulli},
	volume={25},
	number={4A},
	pages={2620--2648},
	year={2019},
	publisher={JSTOR}
}

@article{dudley1969speed,
	title={The speed of mean Glivenko-Cantelli convergence},
	author={Dudley, Richard Mansfield},
	journal={The Annals of Mathematical Statistics},
	volume={40},
	number={1},
	pages={40--50},
	year={1969},
	publisher={JSTOR}
}

@inproceedings{devlin2019bert,
	title={Bert: Pre-training of deep bidirectional transformers for language understanding},
	author={Devlin, Jacob and Chang, Ming-Wei and Lee, Kenton and Toutanova, Kristina},
	booktitle={Proceedings of the 2019 conference of the North American chapter of the association for computational linguistics: human language technologies, volume 1 (long and short papers)},
	pages={4171--4186},
	year={2019}
}

@inproceedings{salazar2020masked,
	title={Masked language model scoring},
	author={Salazar, Julian and Liang, Davis and Nguyen, Toan Q and Kirchhoff, Katrin},
	booktitle={Proceedings of the 58th annual meeting of the association for computational linguistics},
	pages={2699--2712},
	year={2020}
}

@article{radford2019language,
	title={Language models are unsupervised multitask learners},
	author={Radford, Alec and Wu, Jeffrey and Child, Rewon and Luan, David and Amodei, Dario and Sutskever, Ilya and others},
	journal={OpenAI blog},
	volume={1},
	number={8},
	pages={9},
	year={2019}
}

@book{robinson2010dimensions,
	title={Dimensions, embeddings, and attractors},
	author={Robinson, James C},
	volume={186},
	year={2010},
	publisher={Cambridge University Press}
}

@inproceedings{boissard2014mean,
	title={On the mean speed of convergence of empirical and occupation measures in Wasserstein distance},
	author={Boissard, Emmanuel and Le Gouic, Thibaut},
	booktitle={Annales de l’Institut Henri Poincar{\'e}-Probabilit{\'e}s et Statistiques},
	volume={50},
	number={2},
	pages={539--563},
	year={2014}
}

@article{ciampiconi2023survey,
	title={A survey and taxonomy of loss functions in machine learning},
	author={Ciampiconi, Lorenzo and Elwood, Adam and Leonardi, Marco and Mohamed, Ashraf and Rozza, Alessandro},
	journal={arXiv preprint arXiv:2301.05579},
	year={2023}
}

@inproceedings{mroueh2023towards,
	title={Towards a statistical theory of learning to learn in-context with transformers},
	author={Mroueh, Youssef},
	booktitle={NeurIPS 2023 Workshop Optimal Transport and Machine Learning},
	year={2023}
}

@article{vuckovic2021regularity,
	title={On the regularity of attention},
	author={Vuckovic, James and Baratin, Aristide and Combes, Remi Tachet des},
	journal={arXiv preprint arXiv:2102.05628},
	year={2021}
}

@inproceedings{sander2022sinkformers,
	title={Sinkformers: Transformers with doubly stochastic attention},
	author={Sander, Michael E and Ablin, Pierre and Blondel, Mathieu and Peyr{\'e}, Gabriel},
	booktitle={International Conference on Artificial Intelligence and Statistics},
	pages={3515--3530},
	year={2022},
	organization={PMLR}
}

@inproceedings{
xie2022an,
title={An Explanation of In-context Learning as Implicit Bayesian Inference},
author={Sang Michael Xie and Aditi Raghunathan and Percy Liang and Tengyu Ma},
booktitle={International Conference on Learning Representations},
year={2022},
url={https://openreview.net/forum?id=RdJVFCHjUMI}
}

@inproceedings{muller2022transformers,
title={Transformers Can Do Bayesian Inference},
author={Samuel M{\"u}ller and Noah Hollmann and Sebastian Pineda Arango and Josif Grabocka and Frank Hutter},
booktitle={International Conference on Learning Representations},
year={2022},
url={https://openreview.net/forum?id=KSugKcbNf9}
}

@inproceedings{
panwar2024incontext,
title={In-Context Learning through the Bayesian Prism},
author={Madhur Panwar and Kabir Ahuja and Navin Goyal},
booktitle={The Twelfth International Conference on Learning Representations},
year={2024},
url={https://openreview.net/forum?id=HX5ujdsSon}
}

@article{narayanan2010sample,
	title={Sample complexity of testing the manifold hypothesis},
	author={Narayanan, Hariharan and Mitter, Sanjoy},
	journal={Advances in neural information processing systems},
	volume={23},
	year={2010}
}

@article{fefferman2016testing,
	title={Testing the manifold hypothesis},
	author={Fefferman, Charles and Mitter, Sanjoy and Narayanan, Hariharan},
	journal={Journal of the American Mathematical Society},
	volume={29},
	number={4},
	pages={983--1049},
	year={2016}
}

@inproceedings{popeintrinsic,
	title={The Intrinsic Dimension of Images and Its Impact on Learning},
	author={Pope, Phil and Zhu, Chen and Abdelkader, Ahmed and Goldblum, Micah and Goldstein, Tom},
	booktitle={International Conference on Learning Representations},
	year={2021},
	url={https://openreview.net/forum?id=XJk19XzGq2J}
}

@article{nakada2020adaptive,
	title={Adaptive approximation and generalization of deep neural network with intrinsic dimensionality},
	author={Nakada, Ryumei and Imaizumi, Masaaki},
	journal={Journal of Machine Learning Research},
	volume={21},
	number={174},
	pages={1--38},
	year={2020}
}

@article{brown2020language,
	title={Language models are few-shot learners},
	author={Brown, Tom and Mann, Benjamin and Ryder, Nick and Subbiah, Melanie and Kaplan, Jared D and Dhariwal, Prafulla and Neelakantan, Arvind and Shyam, Pranav and Sastry, Girish and Askell, Amanda and others},
	journal={Advances in neural information processing systems},
	volume={33},
	pages={1877--1901},
	year={2020}
}

@article{samuel2024berts,
	title={BERTs are generative in-context learners},
	author={Samuel, David},
	journal={Advances in Neural Information Processing Systems},
	volume={37},
	pages={2558--2589},
	year={2024}
}

@article{garg2022can,
	title={What can transformers learn in-context? a case study of simple function classes},
	author={Garg, Shivam and Tsipras, Dimitris and Liang, Percy S and Valiant, Gregory},
	journal={Advances in neural information processing systems},
	volume={35},
	pages={30583--30598},
	year={2022}
}

@inproceedings{von2023transformers,
	title={Transformers learn in-context by gradient descent},
	author={Von Oswald, Johannes and Niklasson, Eyvind and Randazzo, Ettore and Sacramento, Jo{\~a}o and Mordvintsev, Alexander and Zhmoginov, Andrey and Vladymyrov, Max},
	booktitle={International Conference on Machine Learning},
	pages={35151--35174},
	year={2023},
	organization={PMLR}
}

@article{ahn2023transformers,
	title={Transformers learn to implement preconditioned gradient descent for in-context learning},
	author={Ahn, Kwangjun and Cheng, Xiang and Daneshmand, Hadi and Sra, Suvrit},
	journal={Advances in Neural Information Processing Systems},
	volume={36},
	pages={45614--45650},
	year={2023}
}

@article{hataya2024automatic,
	title={Automatic domain adaptation by transformers in in-context learning},
	author={Hataya, Ryuichiro and Matsui, Kota and Imaizumi, Masaaki},
	journal={arXiv preprint arXiv:2405.16819},
	year={2024}
}

@inproceedings{giannou2023looped,
	title={Looped transformers as programmable computers},
	author={Giannou, Angeliki and Rajput, Shashank and Sohn, Jy-yong and Lee, Kangwook and Lee, Jason D and Papailiopoulos, Dimitris},
	booktitle={International Conference on Machine Learning},
	pages={11398--11442},
	year={2023},
	organization={PMLR}
}

@article{fu2024transformers,
	title={Transformers learn to achieve second-order convergence rates for in-context linear regression},
	author={Fu, Deqing and Chen, Tian-Qi and Jia, Robin and Sharan, Vatsal},
	journal={Advances in Neural Information Processing Systems},
	volume={37},
	pages={98675--98716},
	year={2024}
}

@article{bai2023transformers,
	title={Transformers as statisticians: Provable in-context learning with in-context algorithm selection},
	author={Bai, Yu and Chen, Fan and Wang, Huan and Xiong, Caiming and Mei, Song},
	journal={Advances in neural information processing systems},
	volume={36},
	pages={57125--57211},
	year={2023}
}

@article{zhang2024trained,
	title={Trained transformers learn linear models in-context},
	author={Zhang, Ruiqi and Frei, Spencer and Bartlett, Peter L},
	journal={Journal of Machine Learning Research},
	volume={25},
	number={49},
	pages={1--55},
	year={2024}
}

@inproceedings{li2023transformers,
	title={Transformers as algorithms: Generalization and stability in in-context learning},
	author={Li, Yingcong and Ildiz, Muhammed Emrullah and Papailiopoulos, Dimitris and Oymak, Samet},
	booktitle={International conference on machine learning},
	pages={19565--19594},
	year={2023},
	organization={PMLR}
}

@article{kim2024transformers,
	title={Transformers are minimax optimal nonparametric in-context learners},
	author={Kim, Juno and Nakamaki, Tai and Suzuki, Taiji},
	journal={Advances in Neural Information Processing Systems},
	volume={37},
	pages={106667--106713},
	year={2024}
}

@article{ching2026efficient,
	title={Efficient and minimax-optimal in-context nonparametric regression with transformers},
	author={Ching, Michelle and Popescu, Ioana and Smith, Nico and Ma, Tianyi and Underwood, William G and Samworth, Richard J},
	journal={arXiv preprint arXiv:2601.15014},
	year={2026}
}

@article{ma2025provable,
	title={Provable test-time adaptivity and distributional robustness of in-context learning},
	author={Ma, Tianyi and Wang, Tengyao and Samworth, Richard J},
	journal={arXiv preprint arXiv:2510.23254},
	year={2025}
}

@article{wakayama2025context,
	title={In-context learning is provably Bayesian inference: a generalization theory for meta-learning},
	author={Wakayama, Tomoya and Suzuki, Taiji},
	journal={arXiv preprint arXiv:2510.10981},
	year={2025}
}

@article{liu2025context,
	title={In-Context Learning as Nonparametric Conditional Probability Estimation: Risk Bounds and Optimality},
	author={Liu, Chenrui and Tan, Falong and Xie, Chuanlong and Zeng, Yicheng and Zhu, Lixing},
	journal={arXiv preprint arXiv:2508.08673},
	year={2025}
}

@inproceedings{wumany,
	title={How Many Pretraining Tasks Are Needed for In-Context Learning of Linear Regression?},
    author={Jingfeng Wu and Difan Zou and Zixiang Chen and Vladimir Braverman and Quanquan Gu and Peter Bartlett},
    booktitle={The Twelfth International Conference on Learning Representations},
    year={2024},
    url={https://openreview.net/forum?id=vSh5ePa0ph}
    }

@inproceedings{hedeberta,
	title={{DeBERTa}: Decoding-enhanced {BERT} with Disentangled Attention},
	author={He, Pengcheng and Liu, Xiaodong and Gao, Jianfeng and Chen, Weizhu},
	booktitle={International Conference on Learning Representations},
	year={2021},
	url={https://openreview.net/forum?id=XPZIaotutsD}
}

@article{carlier2024quantitative,
	title={Quantitative stability of barycenters in the Wasserstein space: G. Carlier et al.},
	author={Carlier, Guillaume and Delalande, Alex and Merigot, Quentin},
	journal={Probability Theory and Related Fields},
	volume={188},
	number={3},
	pages={1257--1286},
	year={2024},
	publisher={Springer}
}

@article{abedsoltan2024context,
	title={Context-scaling versus task-scaling in in-context learning},
	author={Abedsoltan, Amirhesam and Radhakrishnan, Adityanarayanan and Wu, Jingfeng and Belkin, Mikhail},
	journal={arXiv preprint arXiv:2410.12783},
	year={2024}
}

@article{jiao2026beyond,
  title={Beyond the Prompt in Large Language Models: Comprehension, In-Context Learning, and Chain-of-Thought},
  author={Jiao, Yuling and Lai, Yanming and Lin, Huazhen and Ma, Wensen and Qi, Houduo and Sun, Defeng},
  journal={arXiv preprint arXiv:2603.10000},
  year={2026}
}

@inproceedings{hoffmann2022training,
	title={Training compute-optimal large language models},
	author={Hoffmann, Jordan and Borgeaud, Sebastian and Mensch, Arthur and Buchatskaya, Elena and Cai, Trevor and Rutherford, Eliza and de Las Casas, Diego and Hendricks, Lisa Anne and Welbl, Johannes and Clark, Aidan and others},
	booktitle={Proceedings of the 36th International Conference on Neural Information Processing Systems},
	pages={30016--30030},
	year={2022}
}

@article{wang2023data,
	title={Data management for training large language models: A survey},
	author={Wang, Zige and Zhong, Wanjun and Wang, Yufei and Zhu, Qi and Mi, Fei and Wang, Baojun and Shang, Lifeng and Jiang, Xin and Liu, Qun},
	journal={arXiv preprint arXiv:2312.01700},
	year={2023}
}

@inproceedings{cioba2022distribute,
	title={How to distribute data across tasks for meta-learning?},
	author={Cioba, Alexandru and Bromberg, Michael and Wang, Qian and Niyogi, Ritwik and Batzolis, Georgios and Garcia, Jezabel and Shiu, Da-shan and Bernacchia, Alberto},
	booktitle={Proceedings of the AAAI Conference on Artificial Intelligence},
	volume={36},
	number={6},
	pages={6394--6401},
	year={2022}
}

@article{chakraborty2026generalization,
	title={Generalization Properties of Score-matching Diffusion Models for Intrinsically Low-dimensional Data},
	author={Chakraborty, Saptarshi and Berthet, Quentin and Bartlett, Peter L},
	journal={arXiv preprint arXiv:2603.03700},
	year={2026}
}

@article{raventos2023pretraining,
	title={Pretraining task diversity and the emergence of non-bayesian in-context learning for regression},
	author={Ravent{\'o}s, Allan and Paul, Mansheej and Chen, Feng and Ganguli, Surya},
	journal={Advances in neural information processing systems},
	volume={36},
	pages={14228--14246},
	year={2023}
}

@article{funaki1984certain,
	title={A certain class of diffusion processes associated with nonlinear parabolic equations},
	author={Funaki, Tadahisa},
	journal={Zeitschrift f{\"u}r Wahrscheinlichkeitstheorie und Verwandte Gebiete},
	volume={67},
	number={3},
	pages={331--348},
	year={1984},
	publisher={Springer}
}

@article{leobacher2022well,
	title={Well-posedness and numerical schemes for one-dimensional McKean--Vlasov equations and interacting particle systems with discontinuous drift},
	author={Leobacher, Gunther and Reisinger, Christoph and Stockinger, Wolfgang},
	journal={BIT Numerical Mathematics},
	volume={62},
	number={4},
	pages={1505--1549},
	year={2022},
	publisher={Springer}
}

@article{kawata2026transformers,
  title={Transformers as Measure-Theoretic Associative Memory: A Statistical Perspective and Minimax Optimality},
  author={Kawata, Ryotaro and Suzuki, Taiji},
  journal={arXiv preprint arXiv:2602.01863},
  year={2026}
}

@article{furuya2026approximation,
  title={Approximation Theory for Lipschitz Continuous Transformers},
  author={Furuya, Takashi and Murari, Davide and Sch{\"o}nlieb, Carola-Bibiane},
  journal={arXiv preprint arXiv:2602.15503},
  year={2026}
}

@inproceedings{wang2025can,
  title={Can in-context learning really generalize to out-of-distribution tasks?},
  author={Wang, Qixun and Wang, Yifei and Ying, Xianghua and Wang, Yisen},
  booktitle={International Conference on Learning Representations},
  volume={2025},
  pages={83553--83574},
  year={2025}
}

@article{goddard2025can,
  title={When can in-context learning generalize out of task distribution?},
  author={Goddard, Chase and Smith, Lindsay M and Ngampruetikorn, Vudtiwat and Schwab, David J},
  journal={arXiv preprint arXiv:2506.05574},
  year={2025}
}

@inproceedings{
kwon2026outofdistribution,
title={Out-of-Distribution Generalization of In-Context Learning: A Low-Dimensional Subspace Perspective},
author={Soo Min Kwon and Alec S. Xu and Can Yaras and Laura Balzano and Qing Qu},
booktitle={The 29th International Conference on Artificial Intelligence and Statistics},
year={2026},
url={https://openreview.net/forum?id=xrmPHv8SNT}
}

@article{belkin2019reconciling,
  title={Reconciling modern machine-learning practice and the classical bias--variance trade-off},
  author={Belkin, Mikhail and Hsu, Daniel and Ma, Siyuan and Mandal, Soumik},
  journal={Proceedings of the National Academy of Sciences},
  volume={116},
  number={32},
  pages={15849--15854},
  year={2019},
  publisher={National Academy of Sciences}
}

@article{chakraborty2025statistical,
  title={On the statistical properties of generative adversarial models for low intrinsic data dimension},
  author={Chakraborty, Saptarshi and Bartlett, Peter L},
  journal={Journal of Machine Learning Research},
  volume={26},
  number={111},
  pages={1--57},
  year={2025}
}

@article{canas2012learning,
	title={Learning probability measures with respect to optimal transport metrics},
	author={Canas, Guillermo and Rosasco, Lorenzo},
	journal={Advances in neural information processing systems},
	volume={25},
	year={2012}
}

@article{block2022intrinsic,
	title={Intrinsic dimension estimation using Wasserstein distance},
	author={Block, Adam and Jia, Zeyu and Polyanskiy, Yury and Rakhlin, Alexander},
	journal={Journal of Machine Learning Research},
	volume={23},
	number={313},
	pages={1--37},
	year={2022}
}

@article{jiao2023deep,
	title={Deep nonparametric regression on approximate manifolds: Nonasymptotic error bounds with polynomial prefactors},
	author={Jiao, Yuling and Shen, Guohao and Lin, Yuanyuan and Huang, Jian},
	journal={The Annals of Statistics},
	volume={51},
	number={2},
	pages={691--716},
	year={2023},
	publisher={Institute of Mathematical Statistics}
}
\end{document}